\documentclass[runningheads]{llncs}

\usepackage{eccv}

\usepackage{eccvabbrv}
\usepackage{graphicx}
\usepackage{tikz}
\usepackage{multirow}
\usepackage{comment}
\usepackage{mathtools}

\usepackage{amsthm}
\usepackage{wrapfig}
\usepackage{enumitem}
\usepackage{booktabs}
\usepackage{bm}
\usepackage{color}
\usepackage{fancyvrb}
\usepackage{gensymb}
\usepackage{algorithm,algpseudocode}
\algrenewcommand\algorithmicrequire{\textbf{Input:}}
\algrenewcommand\algorithmicensure{\textbf{Output:}}
\usepackage{stackengine}
\stackMath

\usepackage[accsupp]{axessibility}  

\usepackage{hyperref}
\usepackage[capitalize]{cleveref}
\usepackage{orcidlink}

\definecolor{ben}{rgb}{0.9,0.,0.5}

\title{Alignist: CAD-Informed Orientation Distribution Estimation by Fusing Shape and Correspondences} 
\titlerunning{Alignist:CAD-Informed Orientation Distribution Estimation}

\author{Shishir Reddy Vutukur\inst{1}\orcidlink{0000-0002-4406-8491} \and
Rasmus Laurvig Haugaard\inst{2}\orcidlink{0000-0002-5241-5529} 
\and  Junwen Huang\inst{1} \orcidlink{0000-0002-0504-2395} \and \newline
Benjamin Busam\inst{1,3}\orcidlink{0000-0002-0620-5774} \and
Tolga Birdal\inst{4} \orcidlink{0000-0001-7915-7964}
}

\authorrunning{Vutukur et al.}

\institute{$^1$Technical University of Munich 
 \quad $^2$University of Southern Denmark \newline $^3$Munich Center for Machine Learning (MCML) \quad $^4$Imperial College London }

\renewcommand{\vec}[1]{\mathrm{vec}{(#1)}}
\newcommand{\mat}[1]{\mathbf{#1}}

\newcommand{\x}{\mathbf{x}}

\newcommand{\R}{\mathbb{R}}

\newcommand{\rb}{\mathbf{r}}

\newcommand{\given}[1][]{\:#1\vert\:}
\newcommand{\E}{\mathbb{E}}

\newcommand{\C}{\mathbf{C}}

\newcommand{\Id}{\mathbf{I}}

\newcommand{\Rot}{\mathbf{R}}
\newcommand{\X}{\mathbf{X}}

\newcommand{\Zbb}{\mathbb{Z}}

\newcommand{\tb}{\mathbf{t}}

\newcommand{\ab}{\mathbf{a}}

\newcommand{\J}{\mathbf{J}}

\newcommand{\pt}{\mathbf{p}}

\newcommand{\Img}[0]{\mat{I}} 

\newcommand{\mur}{\bm{\mu}}
\newcommand{\muc}{\bm{\nu}}

\DeclarePairedDelimiterX{\infdivx}[2]{(}{)}{%
  #1\;\delimsize\|\;#2%
}

\newcommand{\Mesh}{\mathcal{M}}

\newcommand{\Tri}{\mathbf{T}}

\newcommand{\params}{\bm{\theta}}
\newcommand{\paramstwo}{\bm{\phi}}

\newcommand*\diff{\mathop{}\!\mathrm{d}}

\crefname{equation}{eq.}{eq.}
\Crefname{equation}{Eq.}{Eq.}
\crefname{theorem}{thm.}{thms.}
\Crefname{Theorem}{Thm.}{Thms.}
\crefname{conjecture}{conj.}{conjs.}
\Crefname{Conjecture}{Conj.}{Conjs.}
\crefname{proposition}{prop.}{props.}
\Crefname{proposition}{Prop.}{Props.}
\crefname{definition}{dfn.}{dfn.}
\Crefname{definition}{Dfn.}{Dfn.}
\crefname{remark}{remark}{remark}
\Crefname{Remark}{Remark}{Remark}
\Crefname{algorithm}{Alg.}{Alg.}

\newtheorem{prop}{Proposition}

\crefname{section}{Sec.}{Secs.}
\Crefname{section}{Sec.}{Secs.}
\crefname{equation}{Eq.}{Eqs.}
\Crefname{equation}{Eq.}{Eqs.}
\crefname{figure}{Fig.}{Figs.}
\Crefname{figure}{Fig.}{Figs.}
\crefname{table}{Tab.}{Tabs.}
\Crefname{table}{Tab.}{Tabs.}
\crefname{thm}{Thm.}{Thms.}
\Crefname{thm}{Thm.}{Thms.}
\crefname{conj}{Conj.}{Conjs.}
\Crefname{conj}{Conj.}{Conjs.}
\crefname{dfn}{Dfn.}{Dfns.}
\crefname{dfn}{Dfn.}{Dfns.}
\crefname{remark}{remark}{remarks}
\Crefname{Remark}{Remark}{Remarks}
\crefname{prop}{Prop.}{Prop.}
\Crefname{prop}{Prop.}{Prop.}
\Crefname{algorithm}{Alg.}{Alg.}
\crefname{appendix}{App.}{apps.}
\Crefname{appendix}{App.}{Apps.}
\crefname{appsec}{appendix}{appendices}
\Crefname{appsec}{Appendix}{Appendices}

\DeclareMathOperator*{\argmin}{arg\min}
\DeclareMathOperator*{\argmax}{arg\max}

\renewcommand{\paragraph}[1]{{\vspace{0.6mm}\noindent \bf #1}.}

\begin{document}

\maketitle
\begin{abstract}
Object pose distribution estimation is crucial in robotics for better path planning and handling of symmetric objects. Recent distribution estimation approaches employ contrastive learning-based approaches by maximizing the likelihood of a single pose estimate in the absence of a CAD model. We propose a pose distribution estimation method leveraging symmetry respecting correspondence distributions and shape information obtained using a CAD model. Contrastive learning-based approaches require an exhaustive amount of training images from different viewpoints to learn the distribution properly, which is not possible in realistic scenarios. Instead, we propose a pipeline that can leverage correspondence distributions and shape information from the CAD model, which are later used to learn pose distributions. Besides, having access to pose distribution based on correspondences before learning pose distributions conditioned on images, can help formulate the loss between distributions. The prior knowledge of distribution also helps the network to focus on getting sharper modes instead. With the CAD prior, our approach converges much faster and learns distribution better by focusing on learning sharper distribution near all the valid modes, unlike contrastive approaches, which focus on a single mode at a time. We achieve benchmark results on SYMSOL-I and T-Less datasets.

\end{abstract}    
\section{Introduction}
Real world is quintessentially VUCA, \ie, volatile, uncertain, complex and ambiguous~\cite{johansen2013navigating}, significantly complicating the task of any methodology aimed at perceiving and interpreting our natural environments.
In fact, this inherent nature renders the quest for a universally applicable, unique solution to computer vision, elusive. Still, modern AR/VR systems and robotics applications require to navigate in our VUCA world with a high degree of accuracy, robustness, safety and reliability~\cite{placed2023survey}. In such cases, it becomes more pragmatic to reason about uncertainty and ambiguities rather than seeking a one-size-fits-all solution.

In the realm of computer vision, uncertainty often stems from incomplete, insufficient, or noisy data, while ambiguity arises when visual information can be interpreted in multiple ways. 
For example, a shadow might be mistaken for an object, or a partially occluded object might be hard to identify. Both of these notions can be captured by a \emph{multi-modal probability distribution}, whose modes correspond to ambiguous yet probable solutions, and the variation captures a notion of uncertainty. 
This motivated a plethora of pose estimation approaches~\cite{deng2022deep,murphy2021implicit,haugaard2023spyropose} to infer densities over the pose space rather than predicting point estimates as done classically~\cite{hinterstoisser2013model,xiang2017posecnn}.

\begin{figure}[t!]
        \centering
\includegraphics[width=\textwidth]{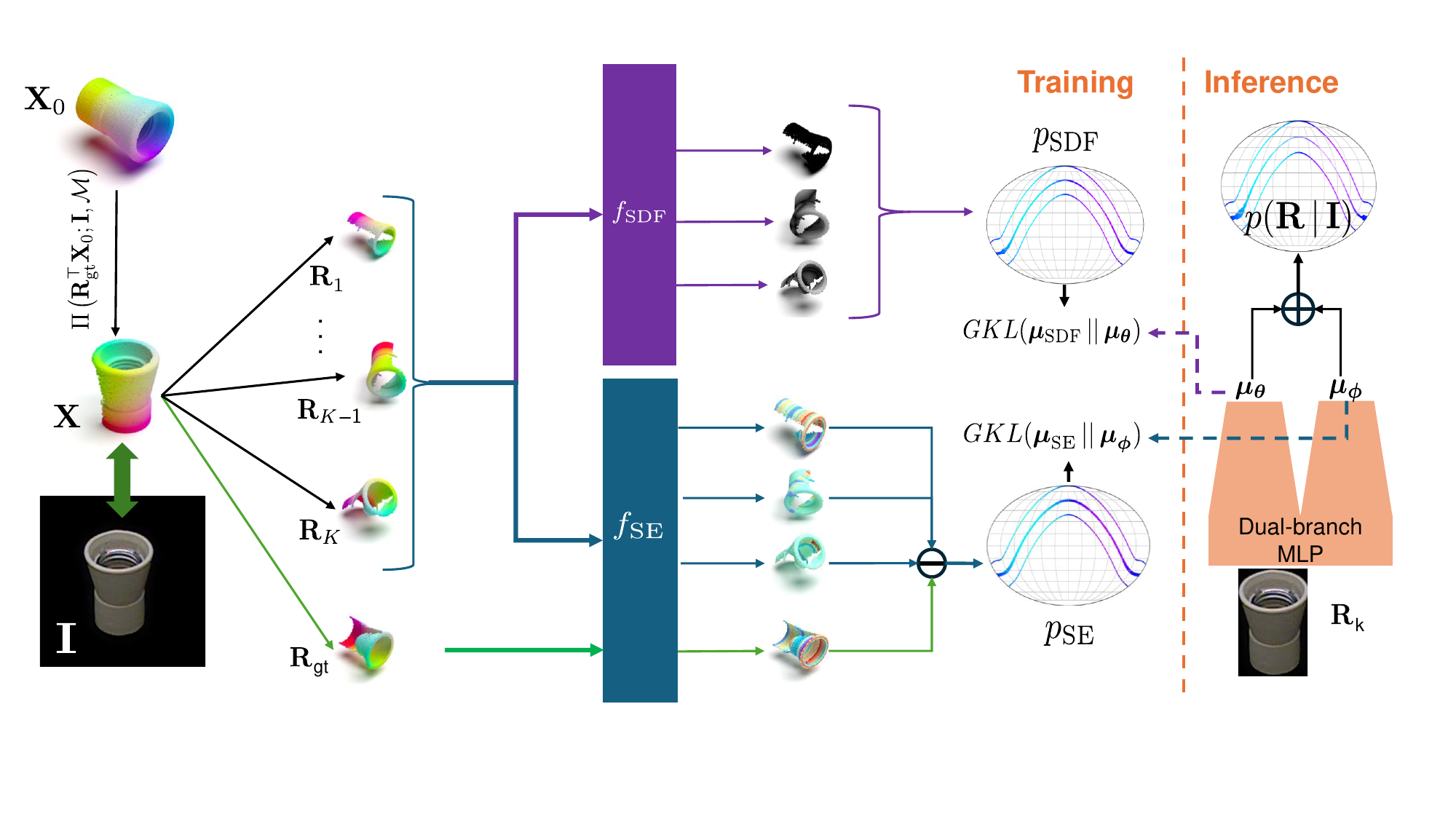}
        \caption{Training and Inference Pipeline: We employ a training mechanism where supervision is generated from pre-trained SurfEmb($f_\mathrm{SE}$) and SDF( $f_\mathrm{SDF}$) blocks. The CAD model, $\X_0$, undergoes a projection to render an image aligned point cloud, $\X$. The image aligned point cloud, $\X$, is rotated with ground truth rotation, $\Rot_{gt}$, and passed through $f_\mathrm{SE}$ block to estimate canonical features. Similarly, $\X$, is rotated with a random rotations, $\Rot_{k}$, and passed through $f_\mathrm{SE}$ block to generate features that are compared with canonical features to estimate the score $\mur_\mathrm{SE}$ for the rotation. Similarly, the rotated point cloud with a random rotation is passed through $f_\mathrm{SDF}$ to estimate the SDF values of the point cloud. An $L_0$ norm is applied to the SDF values to compute $\mur_\mathrm{SDF}$ score for the rotation. These scores are used to supervise the Dual-branch MLP. The Dual-branch MLP network takes an image and the same rotation matrix, $\Rot_{k}$, as input and infers two scores $\mur_\mathrm{\theta}$ and $\mur_\mathrm{\phi}$. This process is carried out for $K$ rotations for a given image and a Generalized KL divergence loss (GKL) is formulated between inferred scores from the right block and estimated scores from the left block to train the Dual-branch MLP network. The Dual-branch MLP is part of both training and inference. During inference, an image and rotations sampled from a grid are passed through the network to estimate the full distribution on the grid.}
        \label{fig:Teaser} 
\end{figure}

The rotational components of the 6D pose are particularly impacted by ambiguities, which mostly arise due to symmetry or self-occlusions~\cite{manhardt2019explaining}. Consequently, state-of-the-art approaches aim to characterize the density over $SO(3)$ (rotations), either directly in the form of mixtures of matrix distributions~\cite{deng2022deep,yin2023laplace,yin2022fishermatch} or through implicit neural fields~\cite{murphy2021implicit} and invertible neural networks~\cite{liu2023delving}.
These methods generally assume that the training data consists of posed images, and, in one way or another, learn to infer densities conditioned on an input image. Such an approach necessitates a large amount of training data, as the networks need to encounter the same appearance under various rotations.

In this paper, we introduce, \textbf{Alignist}, a new way of learning distributions over $SO(3)$, leveraging the CAD model that is readily available for a majority of 3D pose tasks and provided with typical datasets together with object poses either in instance or category form. 
To make use of this additional piece of shape information, we first establish the proportionality of image-conditional densities over rotations to densities over the visible parts of the CAD model. With this change of variables, we can sample the space of rotations during training and obtain a sharp and full density over the entire $SO(3)$ in the form of an unnormalized, empirical measure (scores). In particular, we model the desired density over the model as a \emph{products of experts}~\cite{hinton1999products}, reflecting the spatial and feature-space alignment to the original CAD model. We choose two expert probabilities to follow Bolzmann densities induced by (i) the distance to the original CAD model measured by the of norm of signed distances, (ii) the similarity of symmetry-respecting features obtained from SurfEMB~\cite{haugaard2022surfemb}. The first expert specifies to 3D geometric alignment whereas the second one is additionally informed by the appearance, as SurfEMB is trained also using the image cues. The final product-probability supervises a dual-branch MLP, which, given an input image, infers both of the distributions in the form of empirical distributions or measures by sampling over $SO(3)$. As a loss between the computed and inferred distributions, we utilize the recently proposed \emph{generalized KL divergence}~\cite{miller2023simulation}, capable of comparing unnormalized probability measures.
This training scheme, illustrated in \cref{fig:Teaser}, is agnostic to the network architecture and not only facilitates stronger cues for learning but also aids in generalization in a low-data regime, as the available CAD model is exploited to generate densely sampled distributions.
Lastly, as classical positional encoding methods do not generalize to the manifold of $SO(3)$~\cite{esteves2021generalized}, we introduce a new rotation encoding method by transforming the vertices of a unit cube under the given rotation.

In summary, our contributions are:
\begin{itemize}[noitemsep,leftmargin=*]
    \item To the best of our knowledge, we present the first method to use a CAD model to train an implicit network in order to infer densities over $SO(3)$.
    \item We propose to utilize geometry-aware (SDF) and symmetry-aware (SurfEMB) experts to obtain rich and informative supervision cues during training.
    \item Backed by a novel positional encoding, we propose a dual-branch MLP to infer two distributions, one on the SDF and one on the surface features.

    \item Experiments demonstrate the advantages of our approach in capturing ambiguities and uncertainties on textureless objects, especially in low-data regimes, thanks to powerful 3D object priors our method utilizes.
\end{itemize}
We make our implementation publicly available \href{https://github.com/shishirreddy/Alignist}{here}.

\section{Related Work}

While a plethora of 6D object pose estimation exist~\cite{hinterstoisser2013model,xiang2017posecnn,wang2021gdr,welsa, nerfpose, nerfeat, matchu}, we will specifically review works which deal with rotational symmetries in the context of deep learning. 

\paragraph{Rotation representations for deep networks}
Deep neural networks typically produce feature vectors in a Euclidean space, making it hard to reason on nonlinear manifolds, such as that of $SO(3)$. Early deep learning models opted for rather classical parameterizations for regressing rotations, like Euler-angles \cite{renderforcnn15, 3d-rcnn18}, direction cosine matrices (DCM)~\cite{huang2021multibodysync,yi2019deep}, axis-angle \cite{DeMoN17, deep-6dpose18}, and quaternions~\cite{xiang2017posecnn,posenet,zhao2020quaternion}. 
Some scholars~\cite{kehl2017ssd,busam2020like,su2022zebrapose} instead discretize the space and classify, however, pose space represents a continuous Riemannian manifold~\cite{busam2017camera}.
Recently,~\cite{zhou2019continuity} showed that the typical representations are \emph{discontinuous}, \ie, do not admit homeomorphic maps to $SO(3)$. Instead, 6D~\cite{zhou2019continuity}, 9D~\cite{levinson2020analysis} and even 10D~\cite{peretroukhin_so3_2020} representations were proposed to resolve the discontinuity issue and improve the regression accuracy. Br{\'{e}}gier~\cite{regression_manifold21} has thoroughly examined different manifold mappings, finding out that SVD orthogonalization~\cite{levinson2020analysis} performs the best when regressing arbitrary rotations.

Besides considering rotations as manifold-valued labels, there are multiple ways they are used in deep networks. Regressing rotations lead to geometric gradients, which require rethinking of backpropagation~\cite{teed2021tangent,chen2022projective}. When used as input data, the sinusoidal basis functions that are used to encode Euclidean coordinates become ill-suited for $SO(3)$~\cite{esteves2021generalized}.
In this work, we also contribute a symmetry-respecting \emph{positional encoding} for rotations, inspired by the metrics in~\cite{bregier2018defining} and also follow a regression approach, however, we investigate pose distribution prediction rather than single pose estimation to incorporate multimodal solutions induced by object or projection ambiguities.

\paragraph{Representing belief over rotations in deep networks}
Extending Bui~\etal~\cite{bui2020eccv} to object pose estimation, Deng~\etal~\cite{deng2022deep}, along with~\cite{Gilitschenski2020Deep}, were the first to address pose ambiguity and uncertainty prediction via deep networks. In these works, as well as extensions such as~\cite{peretroukhin_so3_2020}, a Bingham distribution was used to represent the belief over $SO(3)$. 
The suitability of Bingham distribution was later questioned by follow-up works, proposing Matrix-Fisher distributions~\cite{yin2022fishermatch} or heavy-tailed variants such as Laplace distributions~\cite{yin2023laplace,yin2023towards} for increased robustness.
Implicit-PDF (iPDF)~\cite{murphy2021implicit} used a rotation-conditioned neural network to implicitly learn the orientation distribution and SpyroPose~\cite{haugaard2023spyropose} proposed a coarse-to-fine version of iPDF using a hierarchical grid. HyperPosePDF~\cite{hofer2023hyperposepdf} learned the weights of an iPDF via a hyper-network and spherical convolutions are used in I2S~\cite{klee2022image} to map a deep feature-map onto a distribution over $SO(3)$.
Recently, Liu~\etal~\cite{liu2023delving} proposed a normalizing flow based approach to map an initial distribution to a target one. This approach has later been generalized to manifolds other than $SO(3)$ via free-form flows~\cite{sorrenson2023learning}.

We follow the ideas of implicit distribution learning by incorporating a signed distance function~\cite{park2019deepsdf} alongside a method for ambiguity aware pose description~\cite{haugaard2022surfemb} into our approach.

\paragraph{Uncertainty and ambiguity aware object pose estimation}
Object shape, symmetry, or occlusions can cause ambiguities in object poses given only the perspective projection.
Corona~\etal~\cite{corona2018pose} and Pitteri~\etal~\cite{pitteri2019object} assumed the availability of object symmetry information during training. As obtaining this information is challenging, others tried to explicitly learn it from data~\cite{lee2023object}. State-of-the-art approaches bypass the requirement of prior knowledge about object symmetries.
Manhardt~\etal~\cite{manhardt2019explaining} modeled the ambiguity in object pose estimation by predicting multiple hypotheses, for a given object’s visual appearance, and Shi~\etal~\cite{shi2021fast} utilize an ensemble of object pose estimators to derive uncertainties.
SurfEmb~\cite{haugaard2022surfemb}, EPOS~\cite{hodan2020epos} and NeRF-Feat~\cite{nerfeat} learned a dense distribution of 2D-3D correspondences. EPOS~\cite{hodan2020epos} handled symmetries implicitly by discretizing the surface and predicting a probability distribution over fragments per pixel, whereas SurfEmb, NeRF-Feat estimated a continuous distribution over the object surface. Similarly, Ki-pode~\cite{iversen2022ki} used object keypoints as an intermediary to derive the probability density, explicitly and~\cite{yang2023object} makes use of the \emph{conformal prediction} framework to propagate uncertainties from keypoints to object poses.

In our pipeline, we incorporate a learned continuous distribution~\cite{haugaard2022surfemb} of 2D-3D correspondences between object model and image.
\section{Method}
\paragraph{Problem setting}
We consider an input image crop $\Img\in\R^{M\times N\times 3}$ parameterized by coordinates $\pt\in\R^2$ on the image lattice, containing the projection of a 3D object $\Mesh=(\X,\Tri)$ represented in normalized object coordinate space (NOCS)~\cite{wang2019normalized} $\X=\{\x_i\in [-1,1]^3\ \mid i = 1,\ldots,N_X\}$ together with triangle faces $\Tri=\{\tb_i\in\Zbb^3  \mid i = 1,\ldots,N_T\}$. 
We formulate the rotation estimation problem as a probabilistic inference, where we are interested in the following two quantities:
\begin{enumerate}[itemsep=0.5pt,leftmargin=*,topsep=1pt]
    \item Single-best orientation obtained by the maximum a-posteriori (MAP) estimate: $\Rot^\star = \argmax_{\Rot} p (\Rot\given{\Img}) $, where $p (\Rot\given{\Img})$ models the conditional probability distribution for the orientation of $\Img$.
    \item Multi-modal orientation estimate in the form of a full posterior distribution: 
    $p(\Rot\given{\Img})\propto p(\Img\given{\Rot})\times p(\Rot) = p(\Img, \Rot)$.

\end{enumerate}
A plethora of the 6D object pose estimation methods solve (1) (excluding translations) to get a point estimate~\cite{kehl2017ssd,hodan2018bop,su2022zebrapose}. While solving (2) can provide important additional information, such as
\emph{uncertainty} and \emph{ambiguity}, perhaps not surprisingly it is a much harder task. Note that, in addition to the usual difficulties associated with posterior characterization, in our context we are facing extra challenges due to the Riemannian nature of $SO(3)$~\cite{busam2017camera}. This is arguably the reason why the task of \emph{multiple simultaneous rotation estimations} remains unsolved, to date.
In this work, we focus on estimating posterior densities over rotations, which are the main sources of ambiguity.

\paragraph{Model}
Inspired by the \emph{energy based models}~\cite{lecun2006tutorial,landau2013course}, we write down the posterior density of interest as $p(\Rot\given{\Img})=\exp{\left(-U\left(\Rot\right)\right)}/p(\Rot)$ where $U$ is the \emph{potential energy} with the form $U(\Rot)=-\left(\log p(\Rot\given{\Img})+\log p(\Rot)\right)$.
Estimating $p(\Rot)$ by discretizing the $SO(3)$ by $N_V$
partitions of volume $V$ and marginalizing, leads to the following approximate posterior density~\cite{murphy2021implicit}:
\begin{align}\label{eq:normalizeR}
    p(\Rot\given{\Img}) \approx  \frac{1}{V} {\exp{\left(-U\left(\Rot\right)\right)}} \;/\; {\sum\limits_{i=1}^{N_V} \exp{\left(-U\left(\Rot_i\right)\right)}} = \frac{1}{V}\mathrm{softmax}\left(-U(\Rot)\right).
\end{align}

To design $U(\Rot)$, we assume the availability of a $\X_0$, s.t. $\X^\prime = \Rot\Pi(\Rot_\mathrm{gt}^{\top}\X_0;\Img,\Mesh)$ 

is obtained by rotating the \emph{image-aligned} model $\X$ back to the CAD space using $\Rot$. $\Rot_{\mathrm{gt}}$ is the ground truth pose of the image under consideration and the rendering operator $\X=\Pi(\Rot_{\mathrm{gt}}^\top\X_0;\Img,\Mesh)$ extracts the points of the transformed mesh, visible in the image $\Img$ by rendering the model mesh $\Mesh$.
In order to maintain the invariance under symmetries, \ie, $\Rot\X\equiv\X$ for any \emph{stabilizer} $\Rot$ of $\X$, we relate orientation distributions to distributions over object coordinates:
\begin{prop} {The probability} $p(\Rot\given{\Img})\propto p(\X^\prime \given{\Img})$ where $\X^\prime = \Rot\Pi(\Rot_\mathrm{gt}^{\top}\X_0;\Img,\Mesh)$. 

\end{prop}
\begin{proof}
We will consider the vectorized versions of our variables $\rb:=\vec{\Rot}$ and $\x^\prime:=\vec{\X^\prime}$ and assume $p(\rb\given{\Img})=p(\Rot\given{\Img})$ and $p(\x^\prime\given{\Img})=p(\X^\prime\given{\Img})$. Let $\C=\Pi(\Rot_\mathrm{gt}^{\top}\X_0;\Img,\Mesh)$ 

denote the visible part of the model, {w.l.o.g.} assumed to be a constant. Then, the two distributions can be related by the \emph{change of variables}:
    \begin{align}
        p(\rb\given{\Img})= p(\x^\prime \given{\Img}) \cdot |\J|,
   \end{align}
   where
   \begin{align}
   \J := \frac{\diff \x^\prime}{\diff \rb}=\frac{\diff \vec{\Rot\C}}{\diff\rb} = \frac{\diff \left(\left(\C^\top \otimes \Id_{N_X}\right)\rb\right)}{\diff\rb}=\C^\top \otimes \Id_{N_X},
   \end{align}
   where $\Id_{N_X}$ is the $N_X\times N_X$ identity matrix.
    Hence the determinant of the Jacobian, $|\J|:=|\det\J|$, is independent of $\Rot$:
    \begin{align}
    |\J|=|\C^\top \otimes\Id_{N_X}|=\sqrt{|(\C^\top \otimes\Id_{N_X})^\top(\C^\top \otimes\Id_{N_X})|}=\sqrt{|\C\C^\top|}.
    \end{align}
    As the canonical model {$\Mesh$} (hence $\X_0$) and $\Rot_{\mathrm{gt}}$ are known, $|\J|$ remains constant leading to: 
    \begin{align}
    p(\Rot\given{\Img})= p(\X^\prime \given{\Img}) \cdot \sqrt{|\C\C^\top|}\propto p(\X^\prime \given{\Img}).\nonumber\tag*{\qedhere} 
    \end{align}
   
\end{proof}

With this proportionality, we can instead model $p(\X^\prime \given{\Img})$ in a symmetry aware fashion via the \emph{product of experts} (PoEs)~\cite{hinton1999products}:
\begin{align}\label{eq:pXI}
    p(\X^\prime\given{\Img}) = \frac{1}{\int \prod_i \hat{p}_i(\X^\prime\given{\Img}) \diff\X^\prime}\prod_i \hat{p}_i(\X^\prime\given{\Img})= \frac{1}{Z}\hat{p}_\mathrm{SDF}(\X^\prime\given{\Img})\hat{p}_\mathrm{SE}(\X^\prime\given{\Img}),
\end{align}
where $Z = \int \hat{p}_\mathrm{SDF}(\X^\prime\given{\Img})\hat{p}_\mathrm{SE}(\X^\prime\given{\Img})\diff\X^\prime$.  Note that the product factorizes the conditional probability $p$ into a set of simpler distributions, \emph{experts}, $\hat{p}_i$ which are normalized after multiplication.
The first term, {$\hat{p}_\mathrm{SDF}$}, will measure the deviation from the model in the space of 3D coordinates.
The second term, {$\hat{p}_\mathrm{SE}$}, will ensure that the symmetry-aware features of the transformed model align with the canonical ones under a given $\Rot$ (thus $\X$).
The PoE framework can work with unnormalized densities, $\hat{p}(\cdot)$, allowing us to model $ \hat{p}_\mathrm{SDF}$ and $\hat{p}_\mathrm{SE}$ by a Bolzmann distribution:
\begin{align}\label{eq:Uloss}
\hat{p}_\mathrm{SDF}\propto \exp\left(-\|f_{\mathrm{SDF}}\left(\X^\prime\right)\|_0\right) \,\,\,\mathrm{and}\,\,\,\,\hat{p}_\mathrm{SE}\propto\exp\left(-\|f_{\mathrm{SE}}\left(\X^\prime\right)-f_{\mathrm{SE}}\left(\X_{0}\right)\|_\mathrm{F}\right),
\end{align}
where $f_{\mathrm{SDF}}:\R^3\to\R$ and $f_{\mathrm{SE}}:\X\to\R^{N_F\times N_X}$ are chosen to be the pretrained Deep-SDFs~\cite{park2019deepsdf} and SurfEmb~\cite{haugaard2022surfemb}, respectively. 
Deep-SDF parameterizes the signed distance value to the closest surface point via a neural network. Using $L_0$-norm amplifies the penalty on more distant points. SurfEmb models continuous 2D-3D correspondence $(\X \ni \x \leftrightarrow \pt \in \R^2)$ distributions $p(\x,\pt\given{\Img})=p(\pt\given{\Img})p(\x\given{\Img,\pt})$ over the surface of objects, where $p(\pt\given{\Img})$ denotes the discrete distribution over image coordinates. SurfEmb consists of a network supervised to maximize the probability of the GT coordinates. 
The specific forms of $f_{\mathrm{SDF}}$ and $f_{\mathrm{SE}}$ are precised later in~\cref{sec:network}. 

In our framework, we model these distributions by two MLPs $f_{\params}$ and $f_{\paramstwo}$ with learnable parameters $\params$ and $\paramstwo$, which can be multiplied and passed through the $\mathrm{softmax}$ in~\cref{eq:normalizeR} to yield $p(\Rot\given{\Img})$.

\noindent\textbf{Training.} 
During training, we have access to $\{\Rot^n_{\mathrm{gt}}\}_n$  for all images and thus corresponding $\{\X^\prime_n\}_n$. 
While maximizing $p(\Rot\given{\Img})$ by minimizing the negative log-likelihood for a single ground truth pose is viable, it requires more data samples to cover all symmetries and train the network. 
Another alternative is contrastive training (CT) as done in~\cite{haugaard2023spyropose}. 
Yet, besides the challenge of sampling \emph{hard} negatives over $SO(3)$ during training, CT assumes a single positive label per sample.
This resembles learning a multi-modal distribution by matching several unimodal distributions. Instead, we take a different approach and align the entire distribution. 
To learn the parameters of $(f_{\params},f_{\paramstwo})$ this way, we first obtain two \emph{unnormalized empirical measures} ($\mur_\mathrm{SDF},\mur_\mathrm{SE}$) supported on $SO(3)$, corresponding to $\hat{p}_\mathrm{SDF}$ and $\hat{p}_\mathrm{SE}$, by randomly sampling $SO(3)$ to get the \emph{locations} and computing the individual terms in~\cref{eq:Uloss} as \emph{weights} of these measures. 
Similarly, we let $(\mur_\mathrm{\params},\mur_\mathrm{\paramstwo})$ denote the inferred unnormalized measures (see "Inference") and use the generalized KL divergence~\cite{miller2023simulation} to align unbalanced distributions:

\begin{align}
    \label{eq:GKL}
    \mathrm{GKL}\infdivx{\mur}{\muc} := \sum\limits_{i=1}^m \left(-\log\left(\frac{a^{\muc}_i}{a^{\mur}_i}\right) + \left(\frac{a^{\muc}_i}{a^{\mur}_i}\right) -1\right){a^{\mur}_i},
\end{align}
where $\ab=\{a^{\mur}_i\}$ and $\{\Rot^{\mur}_i\}, i=1,\dots, m$ denote the weights and locations of a discrete measure $\mur = \sum\nolimits_{i=1}^m a_i \delta_{\Rot_i^{\mur}}, a_i\geq 0$, where $\delta_{\Rot^{\mur}}$ is a Dirac delta at $\Rot^{\mur}$. Finally, we train our networks by optimizing the following objectives:
\begin{align}
    \params^\star = \argmin_{\params}\E_{p(\Img)}\mathrm{GKL}\infdivx{\mur_\mathrm{SDF}}{\mur_\mathrm{\params}},
    \quad \paramstwo^\star = \argmin_{\paramstwo}\E_{p(\Img)}\mathrm{GKL}\infdivx{\mur_\mathrm{SE}}{\mur_\mathrm{\paramstwo}},
\end{align}
where $p(\Img)$ denotes the data distribution.
In summarry (see~\cref{fig:Teaser}), to train our networks, we sample the $SO(3)$ randomly or over a grid and use these samples in (i) explicitly computing a product-distribution by leveraging the spatial and feature-space distances and (ii) inferring these conditioned on the image. The divergence between the resulting measures becomes our supervision signal. 

\paragraph{Inference} 
During test time, given an input image $\Img$, we infer the potentially multimodal distribution $p(\Rot\given{\Img})$ by densely sampling $SO(3)$ and evaluating $-\log p(\Rot\given{\Img})= f_{\params}(\Img,\Rot) + f_{\paramstwo}(\Img,\Rot)$. The individual predictions can be thought of as the empirical measures $\mur_{\params}$ and $\mur_{\paramstwo}$, which are combined to get the final measure. We then pass the result through $\mathrm{softmax}$ to arrive at the full posterior. 

\subsection{Network Design, Positional Encoding and Impl. Details}
\label{sec:network}
\paragraph{Network architecture}
Our architecture is similar to IPDF with a ResNet50 image encoder and an MLP each for inferring unnormalized SDF and SE measures. The image encoder encodes the image as a global latent vector which is concatenated with the positional encoded rotation matrix and passed through two MLPs to get unnormalized SDF and SE measures. For a given image, we sample $K$ rotations, to learn the distribution at $K$ discrete points in $SO(3)$ space. To train the network, for a given image, we generate the supervision signal using pre-trained $\mathrm{SE}$ and $\mathrm{SDF}$ MLPs to gather scores for all rotations for supervising the network. The pre-trained $\mathrm{SE}$ and $\mathrm{SDF}$ MLPs are two layer Siren MLPs that predict per-point feature and  $\mathrm{SDF}$ value respectively from a 3D point input.                     

\paragraph{Sampling}
Using our CAD-based priors, SurfEmb and SDF MLPs, we can generate the distribution for viewpoints before we start the training. We sample rotations on an equivolumetric partitioning of $SO3$ grid proposed by Yershova\cite{Yershova2010}. We subdivide the grid up to 5 levels and estimate distribution for certain viewpoints. Using this precomputed distribution, we can sample more rotations near the modes which helps us in learning distribution better. In our approach, instead of estimating distribution for the entire dataset, we estimate distributions for a few viewpoints. It is a tradeoff between memory and accuracy. We save distribution for a few viewpoints at a higher resolution grid than saving distribution for every viewpoint at a lower level grid. This is because the distribution can be transferred to another viewpoint by rotating the grid accordingly using the relative rotation between the given viewpoint and the viewpoint for which the distribution is estimated. We randomly sample a viewpoint from the precomputed distributions and convert it to the current viewpoint and sample the rotations from this distribution. Even though the viewpoints don't exactly have the same distribution in some cases, this won't affect the learning as this will be reflected in the estimated probability and still a bad rotation will be learned as such. Even if we sample rotations away from modes sometimes, the supervision signal will provide a low probability, but it is important to sample more near the modes to learn a sharper distribution. We precompute both shape and feature based distribution for a few viewpoints and use them during training for better sampling. Shape-based precomputed distribution will be valid even in the case of conditionally symmetric objects with ambiguities that break symmetry.

\paragraph{Positional encoding(PE)}
IPDF employs a PE that generates high frequency components of individual rotation matrix elements separately. PE is not applied to the rotation manifold leading to errors in the training where multiple rotations that are far apart in $SO3$ space with similar element-wise absolute values are closer in positional encoding space in some cases. Specifically, when some negative values are lower in magnitude in the rotation matrix, they get a similar PE to when the sign is flipped in another rotation matrix since sines are getting suppressed because of low value and cosine is getting amplified at the same time and also suppressing the sign. The PE is not exactly the same, but the network treats them as similar because the change is minute in cartesian space. To alleviate this issue, we take the 3D corner coordinates of a unit cube and apply positional encoding on the rotated 3D corner coordinates of the cube. The positionally encoded corner coordinates are concatenated to generate a vector that serves as a PE representation of the rotation matrix.  This reduces the noise as shown in Figure \ref{fig:PE} and also brings a boost in accuracy.

\section{Experiments}
We evaluate on Symsol \cite{murphy2021implicit}, T-Less\cite{hodan2017tless}, and ModelNet10-SO3\cite{LiaoCVPR19}. Symsol has two subsets, Symsol-I with untextured symmetric objects and Symsol-II with conditionally symmetric objects with some texture that breaks symmetry from certain viewpoints. We adopt the metrics from IPDF, log-likelihood(LL), and Average Recall(AR) for evaluation. LL measures how strongly the learned network captures the distribution. It is calculated by taking the log mean of the probabilities at the ground truth rotations in our captured distribution. As we can't estimate distribution in continuous rotation space, we compute our distribution on a discrete equivolumetric grid of rotations similar to IPDF. We use 5 levels of subdivision to obtain a rotation grid with around 2M rotations for evaluation. We compute the distribution on the grid and evaluate the LL for rotations at the discrete rotation samples on the grid closest to ground truth rotations. LL evaluation can be carried out with a single ground truth annotation or with multiple ground truth annotation for a single image. To quantify the rotation error on ModelNet, we employ AR with $30^{\circ}$ rotation error threshold. We employ distribution visualization proposed in IPDF where each point on sphere indicates the axis of rotation and the color indicates the tilt about that axis.       
\begin{table}[!htp]\centering
\caption{SYMSOL I results: We present results on Symsol-I evaluated using the log-likelihood metric. Models refers to the number of models trained per dataset. Iterations refer to training iterations and Images refer to the amount of training images employed per object. Note that we follow the convention employed in the Normalizing Flow(NF) paper by adding 2.29 to the scores to make a uniform distribution to have a zero score instead of having -2.29. We adjusted the scores for us and Spyropose(SP) and other approaches were already adjusted in NF paper. We present results for Deng\cite{deng2022deep}, Gil\cite{Gilitschenski}, Prok\cite{prokudin}, IPDF\cite{murphy2021implicit}, SP\cite{haugaard2023spyropose} and NF\cite{liu2023delving}.} \label{tab:symsol_loglik}
\label{tab:symsol_1}
\setlength{\tabcolsep}{2.5pt}
\begin{tabular}{l|ccccccccc|cc}
&Deng &Gil. &Prok. &IPDF &SP &NF &NF &NF &NF &Ours &Ours \\\toprule
Models &1 &1 &1 &1 &5 &1 &1 &1 &1 &1 &1 \\ \hline
Iterations &100k &100k &100k &100k &100k &100k &900k &100k &900k &100k &100k \\ \hline
Images &45k &45k &45k &45k &45k &10k &10k &45k &45k &10k &45k \\ \hline

cone &2.45 &6.13 &-1.05 &6.74 &9.91 &8.45 &8.94 &8.42 &10.05 &9.66 &10.10 \\
cube &-2.15 &0.00 &1.79 &7.10 &10.92 &5.02 &9.01 &7.13 &11.64 &11.29 &12.24 \\
cyl &1.34 &3.17 &1.01 &6.55 &8.75 &8.04 &6.41 &7.83 &9.54 &9.32 &9.40 \\
icosa &-0.16 &0.00 &-0.10 &3.57 &7.52 &-2.14 &-6.03 &2.03 &8.26 &7.99 &9.54 \\
tet &2.56 &0.00 &0.43 &7.99 &10.98 &5.91 &10.79 &8.98 &12.43 &11.39 &11.96 \\\midrule
avg &0.81 &1.86 &0.42 &6.39 &9.62 &5.06 &5.82 &6.88 &10.38 &\textbf{9.69} &\textbf{10.64} \\
\bottomrule
\end{tabular}
\end{table}

\subsection{SYMSOL-I}
Symsol-I comprises textureless geometric shapes, Cube, Tetrahedron, Cone, Icosahedron, and Cylinder that exhibit different types of symmetries and pose a challenging task to capture all the symmetric configurations with high probability. The dataset provides 45k images for each object and they do not provide camera intrinsics or a CAD model. However, these generic CAD models are easily available from any 3D library and are already canonically aligned with the dataset. We employ PBR data\cite{Denninger2023} and SDF samples from CAD models to train  $f_{\mathrm{SE}}, f_{\mathrm{SDF}}$ respectively. We use the frozen $f_{\mathrm{SE}}, f_{\mathrm{SDF}}$ MLPs for training our pipeline. The test set provides ground truth annotations for all the symmetric configurations and is considered in the evaluation for log-likelihood and spread. Prior approaches focus solely on learning distribution from images and show no way to incorporate a CAD model into the training pipeline. Spyropose tries to incorporate a CAD model, but only uses it to strengthen the encoder instead of using a Resnet based encoder. However, their approach requires camera intrinsics. So, they render their version of symsol to use the encoder. Moreover, their cad-based encoder requires the object instance to be known to encode the image.
\setlength{\intextsep}{2pt}
\begin{wrapfigure}{r}{0.37\textwidth}
\centering
\includegraphics[width=0.37\textwidth]{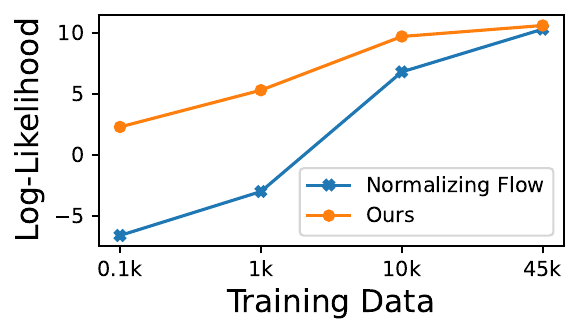}
\caption{LL vs. Training Data for NF, Ours on SYMSOL-I.}
\label{fig:dataplot}
\end{wrapfigure}
To have a common baseline, we compare all the approaches using Resnet image encoder. Using a basic architecture such as IPDF, we achieve benchmark results with log-likelihood scores of  $10.64$ and $9.69$ in full-data and low-data regimes using our training pipeline by distilling knowledge from a CAD model to learn distribution as shown in Table \ref{tab:symsol_loglik}. Our approach converges faster and performs better than NF, SpyroPose with just 100k iterations. In a low-data regime with 10k images, the accuracy drop($0.95$) is not significant depicting the benefits of a CAD prior. Besides, NF has weaker performance with low data and negatively affects some objects with increasing iterations. This demonstrates the advantage of employing a CAD model which is crucial as real data is not easily available in large amounts. Our approach can transfer learned features from the CAD model and help learn the distribution better in the real domain even with fewer images. As shown in Fig.~\ref{fig:dataplot}, the gap between the LLs attained by our method and those by NF increases rapidly as the amount of training data decreases. We present some qualitative results in Fig. \ref{fig:SymsolViz} as proposed in IPDF.

\subsection{SYMSOL-II}
Symsol-II comprises conditionally symmetric objects, tetrahedron, cylinder, and sphere, with markers. So, the modes are reduced when the markers are visible. For conditional symmetric objects, the pose distribution not only shifts but also changes completely based on the visibility of markers. We fit the meshes with silhouette images and estimate intrinsics which is later used to fit the texture of the mesh using color images through differentiable rendering. We train $f_{\mathrm{SE}}, f_{\mathrm{SDF}}$ networks with the obtained mesh. Our approach has slightly weaker performance compared to NF as shown in Table \ref{tab:symsol2}. On the Sphere object, our SDF expert doesn't contribute anything as it learns a uniform distribution and it is a limitation of our approach. Hence, we have to rely only on the SurfEmb expert to explain the distribution which decreases the performance. However, we perform better than NF and Spyropose in the low-data regime with 10k images. Besides, LL metric inherently favors NF as it provides exact LL estimation while IPDF and our approach evaluate LL on a discrete rotation grid, making our methods limited by the finite resolution. The performance drop is less significant on textured cylinder and tetrahedron indicating that the shape component can bring a significant boost to the distribution. LL scores for tetrahedron and cylinder are higher than the Symsol-I dataset since Surfemb features provide information about the texture in these objects and perform better in conjunction with shape experts to capture the textured viewpoints with much finer distribution.

\begin{table}[!t]\centering
\caption{SYMSOL-II results: We report results using log-likelihood metric(LL). We adjust the scores for us and SpyroPose using the convention of normalizing flows(NF) by adding 2.29 to make a uniform distribution to have an average LL of score zero.-10k refers to experiments in low-data regime using only 10k images instead of 45k.}\label{tab:symsol2}
\setlength{\tabcolsep}{3.pt}
\begin{tabular}{l|cccccccc|cc}

Obj &Deng &Gil. &Prok. &IPDF &SP &SP-10k &NF-10k &NF &Ours-10k&Ours \\\hline
SphX &3.41 &5.61 &-1.90 &9.59 &11.36 &7.67 &7.62 &12.37 &6.32 &10.93 \\
cylO &5.28 &7.17 &6.45 &9.20 &11.61 &9.11 &6.99 &12.92 &11.57 &12.18 \\
tetX &5.90 &5.19 &3.77 &10.78 &11.70 &6.48 &3.52 &13.53 &11.53 &12.38 \\\hline
LL &4.86 &5.99 &2.77 &9.86 &11.56 &7.76 &6.04 &\textbf{12.94} &\textbf{9.80} &11.83 \\
\end{tabular}
\end{table}

\begin{figure}
     \centering
     \begin{subfigure}{.3\linewidth}
     \centering
     \includegraphics[width=1\textwidth]{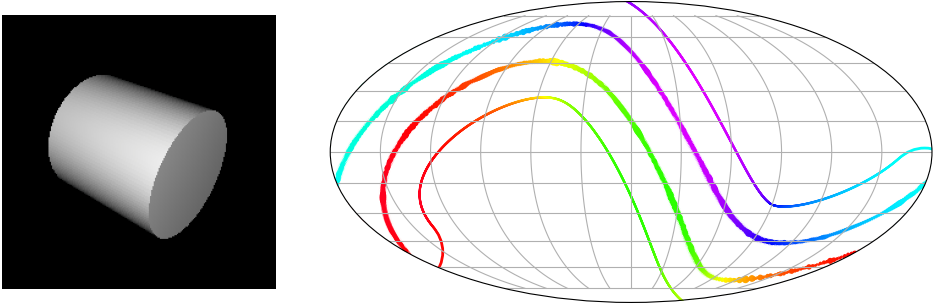}
     \caption{Cyl-Symsol-I\label{fig:cylS1}}
     \end{subfigure}%
     \hfill
     \begin{subfigure}{.3\linewidth}
     \centering
     \includegraphics[width=1.\textwidth]{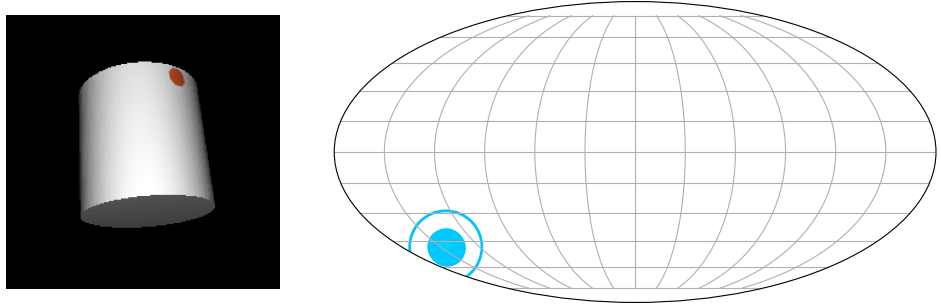}
     \caption{Cyl-Symsol-II\label{fig:cylS2}}
     \end{subfigure}%
     \hfill
     \begin{subfigure}{.3\linewidth}
     \centering
     \includegraphics[width=1\textwidth]{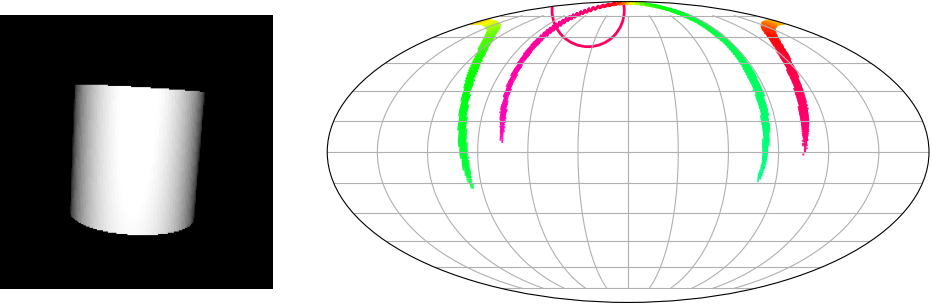}
     \caption{Cyl-Symsol-II\label{fig:cyls2S2}}
     \end{subfigure}%
     \hfill
     \\
     \begin{subfigure}{.3\linewidth}
     \centering
     \includegraphics[width=1\textwidth]{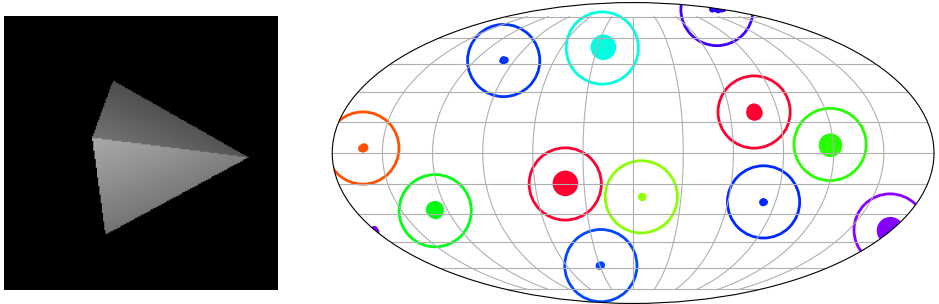}
     \caption{Tet-Symsol-I\label{fig:tetS1}}
     \end{subfigure}%
     \hfill
     \begin{subfigure}{.3\linewidth}
     \centering
     \includegraphics[width=1.\textwidth]{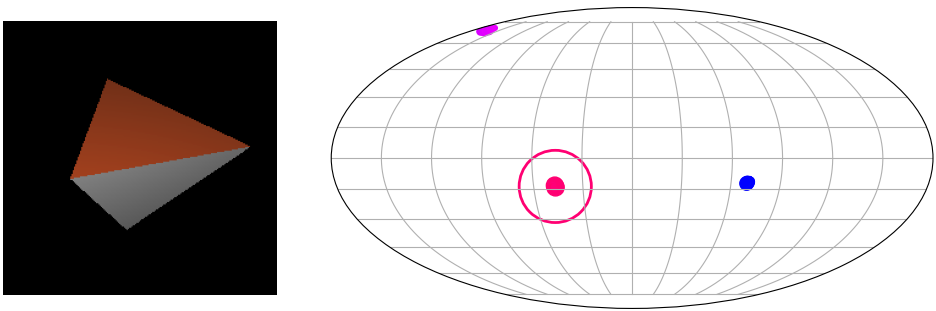}
     \caption{Tet-Symsol-II\label{fig:tetS2}}
     \end{subfigure}%
     \hfill
     \begin{subfigure}{.3\linewidth}
     \centering
     \includegraphics[width=1\textwidth]{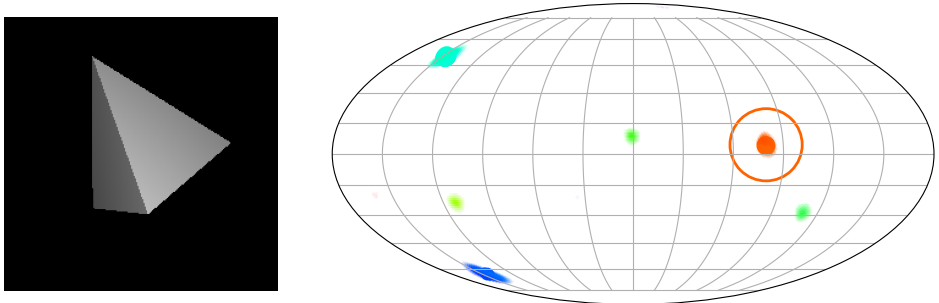}
     \caption{Tet-Symsol-II\label{fig:tet2S2}}
     \end{subfigure}%
     \hfill
     \caption{Pose distribution visualization for objects in Symsol-I and Symsol-II. a) Untextured Cylinder has continuous symmetry b) Textured Cylinder with marker has a unimodal distribution when the marker is visible, c) Broken continuous symmetry on the textured cylinder when the marker is not visible d) Untextured Tetrahedron in Symsol-I has 12 modes which are captured appropriately e) Textured Tetrahedron has three modes when orange face is visible. f) Textured terahedron has 6 modes when the orange face is not visible. Note that only one ground truth annotation is provided in Symsol-II and hence only one mode is circled.  \label{fig:SymsolViz}}
\end{figure} 
\subsection{T-Less}
T-Less contains 30 texture-less symmetric objects. We follow \cite{Gilitschenski} to split data into 75\% for training and 25\% for testing. We train SurfEmb, $f_{\mathrm{SE}}$, using PBR and 75\% of the real data. $f_{\mathrm{SDF}}$ network is trained using the CAD model. We perform better than the benchmark approach, SpyroPose, as shown in Table \ref{tab:tless}. The dataset has minute symmetry breaking features that can turn a multimodal distribution into an unimodal distribution. Although they provide symmetry annotations, the CAD models are not perfectly symmetric like the objects from Symsol-I. So, we follow the other approaches, and the evaluation is performed only on a single ground truth pose similar to Symsol-II. We perform an ablation with training data used for SurfEmb. SurfEmb provides better features with PBR+Real data compared to using just PBR data and is reflected in an increase in LL of 0.2. The results indicate the ability to notice the symmetry breaking features to capture the distribution more sharply when they are visible.  

\begin{table}[!htp]\centering
\caption{T-Less Results: We present T-Less results using the log-likelihood measure. P refers to using only PBR data and P+R refers to using both PBR and Real data for training the Surfemb network. Please note that we still use real data to train our distribution network in all cases. L1 and GKL refer to the loss used for training.}\label{tab:tless}
\setlength{\tabcolsep}{2pt}
\begin{tabular}{l|cccc|cc c}
Method &Prokudin\cite{prokudin} &Gilitschenski\cite{Gilitschenski} &IPDF\cite{murphy2021implicit} &Spyropose\cite{haugaard2023spyropose} &\multicolumn{3}{c}{Ours} \\\hline
SE data&-&-&-&-&P&\multicolumn{2}{c}{P+R} \\\hline
Loss&-&-&-&-&L1&L1&GKL\\\hline
LL &11.0 &9.1 &12.0 &14.1 &13.6 &13.8 &\textbf{14.53}
\end{tabular}
\end{table}

\subsection{ModelNet10-SO3}
ModelNet10-SO3 is an unimodal category-level dataset comprising images of CAD models of 10 categories. We choose a single CAD model per category to train $f_{\mathrm{SE}}, f_{\mathrm{SDF}}$. We employ images from all CADs for training our pipeline, but the distribution supervision comes from the chosen CAD model. The distribution supervision should come ideally from the CAD model of the specific instance in the image, but it is hard to train $f_{\mathrm{SE}}$ for all CADs in the category. We compute the distribution on grid and select the rotation with the highest score for evaluation. We achieve an AR of $70.5\%$ at $30^{\circ}$ rotation error threshold compared to NF($77.4\%$) despite using distribution supervision from a single CAD model per category. Besides, the dataset contains some symmetric models, but only one GT rotation is considered for evaluation leading to lower accuracy.  
\subsection{Ablations}
\paragraph{Positional encoding(PE)} We perform an ablation on different PE of rotation matrices. The PE of elements of rotation matrix leads to noisy distribution as shown in Figure \ref{fig:PE}. This happens when rotation matrices with similar absolute values sometimes get a similar position encoding leading to noise. This could be avoided by parameterizing the encoding of the rotation matrix on manifold by employing Wigner matrices \cite{esteves2021generalized}. We propose a PE where corners of a cube are rotated with rotation matrices and then PE is applied on the transformed 3d coordinates. Both Wigner matrices and our cube position encoding(cube PE) were able to remove noise from the distribution. However, our cube PE performs better compared to Wigner and IPDF PE as shown in Table \ref{tab:Ablations}.
\\

\begin{figure}
     \centering
     \begin{subfigure}{.19\linewidth}
     \centering
     \includegraphics[width=0.5\textwidth]{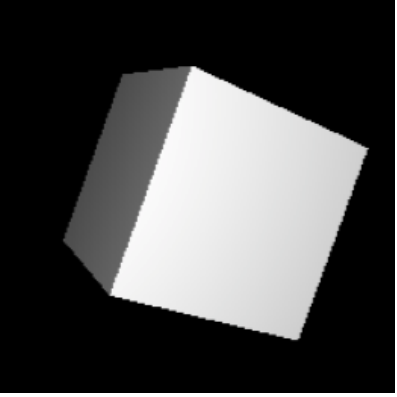}
     \caption{Image\label{fig:failure1}}
     \end{subfigure}%
     \hfill
     \begin{subfigure}{.19\linewidth}
     \centering
     \includegraphics[width=1.\textwidth]{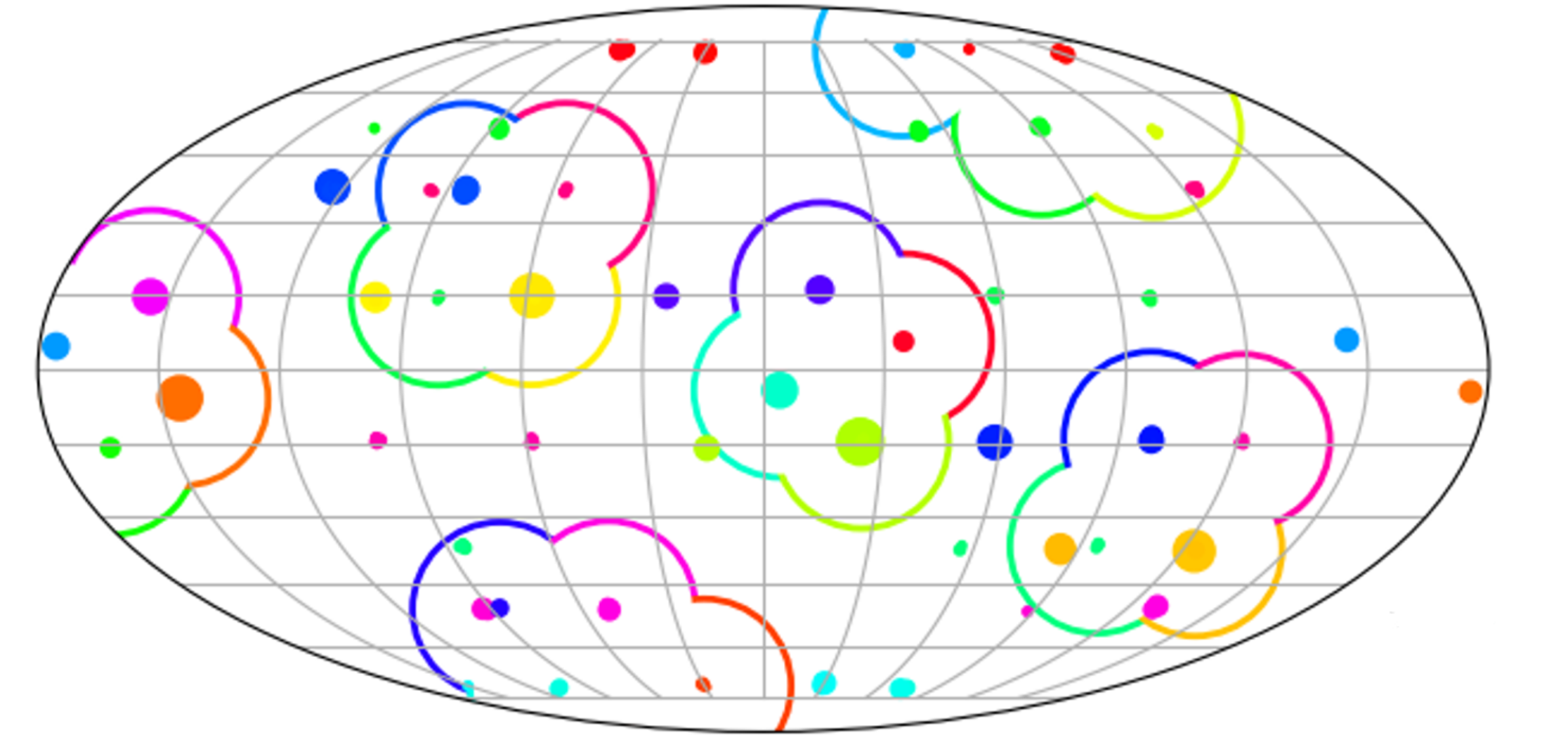}
     \caption{IPDF\label{fig:failure2}}
     \end{subfigure}%
     \hfill
     \begin{subfigure}{.19\linewidth}
     \centering
     \includegraphics[width=1\textwidth]{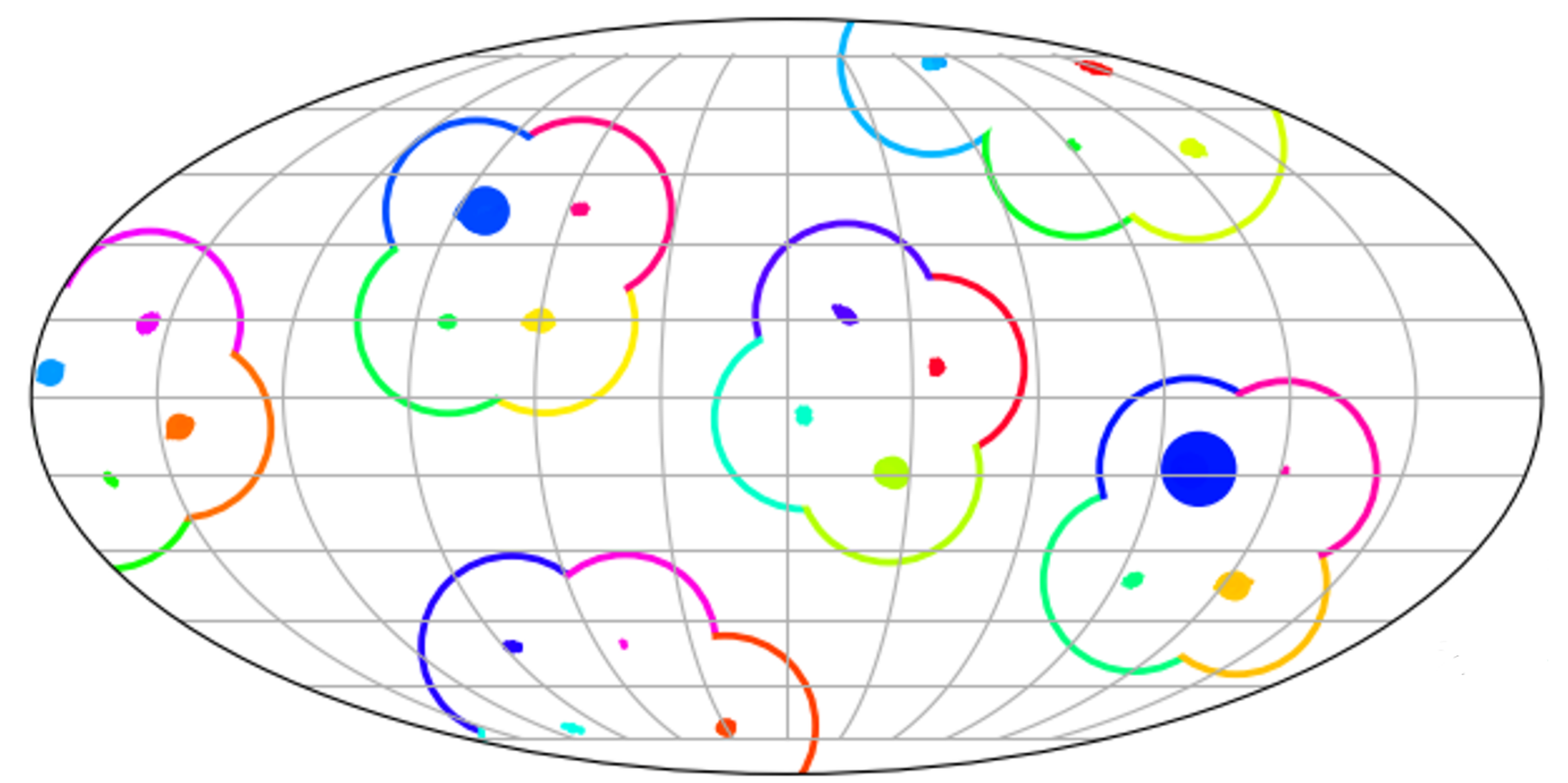}
     \caption{Wigner\label{fig:failure3}}
     \end{subfigure}%
     \hfill
     \begin{subfigure}{.19\linewidth}
     \centering
     \includegraphics[width=1\textwidth]{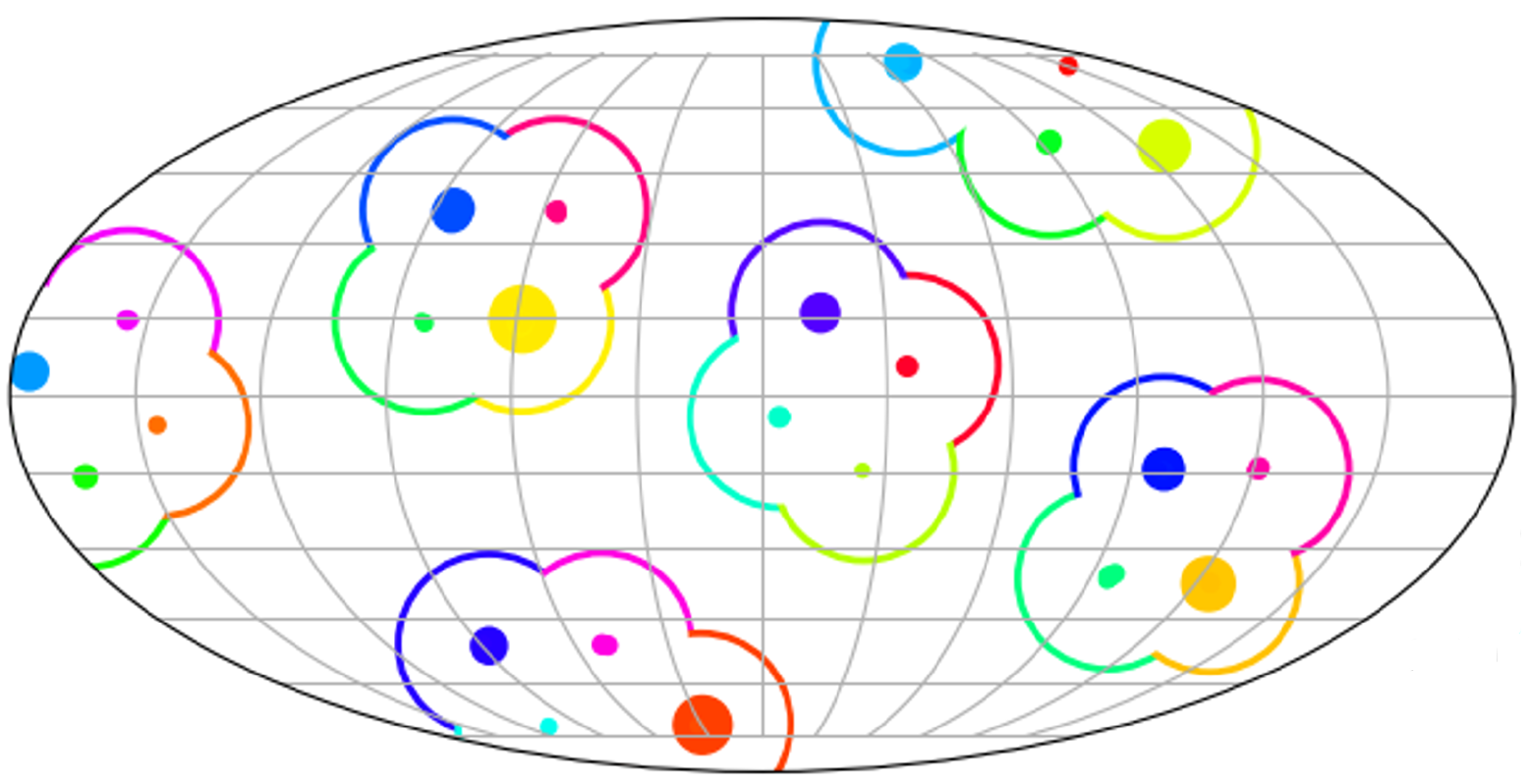}
     \caption{Cube PE\label{fig:failure4}}
     \end{subfigure}
        \begin{subfigure}{0.19\linewidth}
        \centering
         \includegraphics[width=0.5\textwidth]{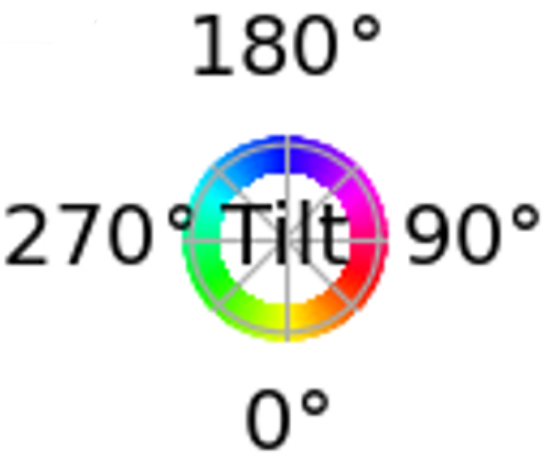}
         \caption{Color Wheel\label{fig:failure5}}
         \end{subfigure}%
     \caption{Pose distribution visualization of our approach with different positional encodings. Our cubePE and wigner matrices based encoding removes noise in the distribution compared to encoding elements of rotation matrices from IPDF. Our encoding removes noise and also has a better performance compared to Wigner matrices. a) Input Image, b)IPDF positional encoding, c) Rotation encoded as Wigner matrices d) Our positional encoding with cube vertices. e) color wheel indicating the tilt about the axis.} \label{fig:PE}
\end{figure} 

\paragraph{Sampling}
During training, we can precompute the distribution for each image through shape and feature based distribution estimation. IPDF employs random sampling that leads to learning a smoother distribution making it difficult to capture modes sharply. As we can convert a single ground truth rotation into the entire distribution, we can sample more rotations near modes and focus on learning the distribution sharply. In Table \ref{tab:Ablations}, random sampling performs worse compared to grid-5 and grid-6 sampling. The grid level refers to HEALPix grid \cite{gorski2005healpix} level at which we estimate our distribution. Sampling on a finer grid leads to improved accuracy. Mode focused sampling is crucial to get sharper modes.

\begin{table}[!t]\centering
\caption{Ablation table: Ablations on SYMSOL-I dataset with positional encoding, sampling strategy, SurfEmb Feature and Shape components, loss terms used for training. Encoding refers to the rotation encoding employed in the approach.  Sampling refers to the strategy employed for sampling rotations during training. Random refers to random sampling. Module refer to module employed to estimate the final distribution. S refers to SDF based module and F refers to the SurfEmb Feature based module. The combined result of SDF and Feat modules provides the best accuracy.} 
\label{tab:Ablations}
\setlength{\tabcolsep}{6pt}
\begin{tabular}{l|cc|cccccccc}
\toprule
Encoding &IPDF &Wigner & \multicolumn{6}{c}{Cube}  \\\hline
Sampling &grid-5 &grid-5 &random &grid-5 &\multicolumn{4}{c}{grid-6}  \\\hline
Module &F+S &F+S &F+S &F+S &S &F &F+S &F+S \\\hline
Loss &L1 &L1 &L1 &L1 &GKL &GKL &L1 &GKL \\\hline
LL &10.09 &8.59 &6.94 &10.2 &10.32 &10.18 &10.48 &10.648 \\
\bottomrule
\end{tabular}
\end{table}

\paragraph{Geometry and SurfEmb based features}
An ablation to see how Shape and SurfEmb feature based probability affect LL in the Symsol-I. LL from the shape and feature modules achieve decent performance separately but perform better when combined to estimate the final distribution. Feature-based probability is appropriate for capturing shape and texture information. Shape based probability doesn't consider texture and it cannot perform alone on Symsol-II. Feature based probability is essential to capture the distribution better in the presence of texture. On the other hand, learning could be affected by the quality of the features from another learned network(SurfEmb) which is not true for shape based probability as it is based on ground truth shape and more reliable for capturing the distribution sharply. In essence, both of the scoring mechanisms together contribute to the betterment of the approach as shown in Table \ref{tab:Ablations}. 

\paragraph{Loss function}
Generalized KL divergence (GKL) performs better than the L1 loss as shown in Table \ref{tab:Ablations}. This is reflected in T-Less dataset as well from Table \ref{tab:tless} with an increase in LL of 0.6. This indicates that the choice of GKL is well suited for formulating loss between unnormalized probability distributions compared to L1 loss, which fails to be a natural fit for comparing distributions.

\section{Conclusion}
We proposed Alignist, a novel approach to learn pose distributions over images by mapping the conditional estimation from a given image to one over a prototypical CAD model, acquired either as ground truth or via a 3D reconstruction. 
Our probabilistic framework uses a product of experts corresponding to losses over SDF and SurfEmb features, and anchors on full distributional distances computed by a generalized KL divergence, rather than considering individual samples as in normalizing flows. Precomputable distribution with the help of the CAD helps in better sampling leading to sharper distribution.
Our novel positional encoding further reduces noise around modes and allows for learning cleaner distributions. We achieve benchmark accuracy in the Symsol-I and T-Less datasets, especially outperforming the state-of-the-art in low data regimes.  


 
\paragraph{Limitations and future work}
Our approach does not explicitly utilize the texture cues either coming from the image or the CAD model. Instead, it exploits the features from SurfEmb which implicitly depend upon both texture and geometry, simultaneously. Employing another expert that models the texture cues explicitly will enable our model to better handle objects like SphereX on Symsol-II where pure geometry cannot contribute towards the distribution. Future work also involves extending our work to diffusion models.

\paragraph{Acknowledgements} T. Birdal acknowledges support from the Engineering and Physical Sciences Research Council [grant EP/X011364/1].



%
%

\bibliographystyle{splncs04}
\bibliography{main}


\title{Supplementary Material \\ Alignist: CAD-Informed Orientation Distribution Estimation by Fusing Shape and Correspondences}
\author{}
\authorrunning{Vutukur et al.}
\institute{}
\maketitle 

We present training details, and additional quantitative and qualitative results in the following sections.
\section{Training Details}
\subsection{Network Architecture}
We employ a ResNet50 encoder similar to other approaches, IPDF\cite{esteves2021generalized}, SpyroPose\cite{haugaard2023spyropose}, and Normalizing Flow\cite{liu2023delving}. We employ two 4-layer MLPs with ReLU activations with 256 dimensional hidden layers each for predicting SDF probability and feat probability. Cube positional encoding comprises encoding transformed vertices of a cube with a rotation matrix. The positional encoding is applied on eight 3D vertices of the cube with 3 frequencies and concatenated to create a 144-dimensional embedding to represent the rotation matrix. The encoder takes the image and predicts a 2048-dimensional image embedding. We employ two linear layers to convert rotation embedding and image embedding to 256 dimensional embeddings which are then added and passed through two MLPs to predict the probabilities. Symsol dataset doesn't provide CAD models. It was easier to get models for Symsol-I as they are readily available in any 3D library and are canonically aligned with the dataset. For Symsol-II, we employed differentiable rendering initially using silhouettes to estimate camera intrinsics because the shapes were canonically aligned. We employ differentiable rendering to transfer texture from images to the CAD models using the learned intrinsics and the rotation labels. 

\subsection{SurfEmb Network}
We follow the default configuration provided in SurfEmb\cite{haugaard2022surfemb} and train using PBR-rendered data from the CAD models for 100k iterations. We extract the SurfEmb feature MLP which is employed in our pipeline to estimate feature based probability. SurfEmb Feature MLP employs a 4-layer Siren MLP which takes a single 3D point and predicts the corresponding feature vector.

\subsection{SDF Network}
We employ a 4-layer Siren MLP with 256 dimensional hidden layers which takes a single 3D point and predicts the corresponding SDF value. We randomly generate 3D points and estimate SDF values for those 3D points using the CAD model. The 3D points and SDF values are used for training the SDF network which remembers the CAD model implicitly by learning SDF values for 3D points. 

\subsection{Training}
We render a NOCS image for every input image using the CAD model and rotation label. We sample 100 pixels inside the mask to extract a 3D point cloud with 100 points. We rotate this with the inverse of ground truth rotation to get the image-aligned point cloud. We employ 4096 rotations to pair with each image to learn the distribution for that specific image. We sample 3000 rotations from the top 20000 rotations in the precomputed distribution. We sample 1095 rotations randomly from a uniform distribution. We also add the ground truth rotation in the employed rotations. We can precompute the scores for every image before the start of training or we can compute them online. We prefer to compute them online so that we can sample randomly without getting restricted to a grid. For each rotation, we compute both SDF probability and feature probability and use them as supervision for probabilities predicted from the network. We employed a batch size of 128 for the network to train on the Symsol dataset and 64 for the T-Less dataset and ModelNet10-SO3. 

\subsection{Mode Focused Sampling}

For each image, we sample 4K rotations to learn the pose distribution. Since we can precompute the distribution for a given training image, we can sample more rotation matrices near the modes to learn a much sharper distribution. Mode focused sampling helps in learning sharper distribution which is possible because of the precomputable distribution from CAD model through SurfEmb features and SDF based shape networks.


\section{ Additional Quantitative results}
\subsection{ModelNet-SO3}
We perform experiments on ModelNet-SO3 to understand the category level generalization capabilities of our approach. ModelNet-SO3 comprises 10 categories which contain synthetic renderings of around 100 CAD models for each category as the training set. The test set comprises synthetic renderings of unseen instances of the same categories. We choose a single CAD model per category which acts as a 3D representation of the category for training our Shape and Feature experts. During training, we employ distribution supervision from the chosen single CAD model per category. While we employ images from all the provided synthetic renderings, the distribution supervision comes from the chosen CAD model. Ideally, we would employ the same CAD model as the instance present in the image, but it is difficult to train SurfEmb and SDF networks for so many CAD models. Despite using supervision from a single CAD model per category, we achieve closer to the benchmark approaches, NF and IPDF.
We presented the results in \cref{tab:modelnet}, training a separate model per category. 

\begin{table}[!t]\centering

\caption{Evaluations on ModelNet-SO3.}
\begin{tabular}{l|ccc}
Metric &IPDF &NF &Ours \\
\midrule
$AR@30^{\circ}$&  0.735 & 0.774 & 0.705 \\

\end{tabular}\label{tab:modelnet}

\end{table}

\subsection{T-Less BOP}
We evaluate IPDF, our approach, and NF on T-less scenes with occlusions by training on synthetic PBR data from the BOP challenge rendered from CAD models. We trained both our approach and NF for 100k iterations. We achieve better log-likelihood compared to NF as shown in in Table \ref{tab:tless-real}. Note that we only trained and evaluated using allocentric rotation representation and ignored translation component. 

\begin{table}[!t]\centering

\caption{Evaluations on occluded real scenes of T-Less BOP using the Log-Likelihood (LL) trained for 100k iterations.}
\begin{tabular}{l|ccc}
Metric &IPDF &NF &Ours \\
\midrule
LL& 5.31&6.23 &\textbf{6.64} \\

\end{tabular}\label{tab:tless-real}

\end{table}
\subsection{Spread on SYMSOL-I}
Spread measures the expected angular deviation from all ground truth symmetric annotations which measures the angular deviation and uncertainty around them. We present spread results on SYMSOL-I dataset in Table \ref{tab:spread}. WE observe that we have lower spread compared to NF at 100k while NF performs better than us after 900k iterations.  
\begin{table}[!t]\centering
\caption{Spread on Symsol-I dataset in degrees. Spread measures the expectation of median angular error over all the ground truth rotations. NF-100k, NF-900k refers to Normalizing flow trained for 100k and 900k iterations respectively. 
}\label{tab:spread}
\scriptsize
\begin{tabular}{lrrrrr}

Spread &Deng &IPDF &NF-100k &NF-900k &Ours \\\midrule
cone &10.1 &1.4 &0.9 &0.5 &0.5 \\
cube &40.7 &4.0 &2.0 &0.6 &0.6 \\
cyl &15.2 &1.4 &0.9 &0.5 &0.5 \\
icosa &29.5 &8.4 &9.2 &1.1 &3.5 \\
tet &16.7 &4.6 &1.7 &0.6 &1.0 \\\hline
avg &22.4 &4.0 &2.9 &\textbf{0.7} &1.4 \\

\end{tabular}
\end{table}

\subsection{Performance on SYMSOL-II}
Our LL on SYMSOL-II is impacted by a domain gap as we had to texture the CAD models using differentiable rendering of images. Moreover, LL in SYMSOL-II considers only a single valid mode(the pose the image is rendered) for evaluation even though there are multiple valid modes. This contrasts with SYMSOL-I, where all valid modes are evaluated. Additionally, LL metric inherently favors NF since NF can provide exact LL estimation. In contrast, both IPDF and our approach evaluate LL on a discrete grid of rotations, making our methods limited by finite resolution.

\section{Efficiency} Inferring rotation distributions takes 10ms at a grid of 2 million rotations on NVIDIA Titan X. NF takes 210ms for rotation estimation, but longer to construct the distribution as the full base distribution is mapped onto the target.

\section{Qualitative results}
We present more distribution visualization results for all objects in the Symsol dataset. We follow the IPDF approach to visualize the top-k probabilities on a 2D sphere where each point on the sphere represents an axis of rotation and the color of the point indicates the tilt about that axis. We visualize distribution on Symsol-I and Symsol-II as described in IPDF in Figures \ref{fig:SymsolViz1} and \ref{fig:SupSymsol2Viz} respectively. We also visualize the SurfEmb features of tetX object in \ref{fig:featViz}. We also visualize distribution by rotating coordinate reference frames with top-k rotations and projecting them on the image as shown in Figure \ref{fig:ourViz}. 
\newpage
\begin{figure}[t]
     \centering
     \begin{subfigure}{.3\linewidth}
     \centering
     \includegraphics[width=1\textwidth]{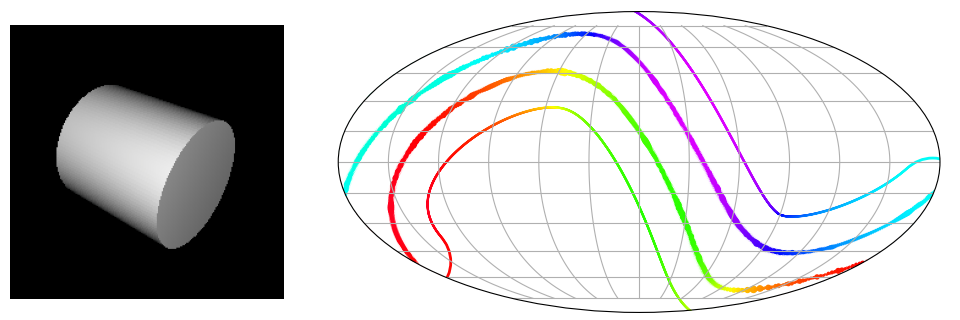}
     \end{subfigure}%
     \hfill
     \begin{subfigure}{.3\linewidth}
     \centering
     \includegraphics[width=1.\textwidth]{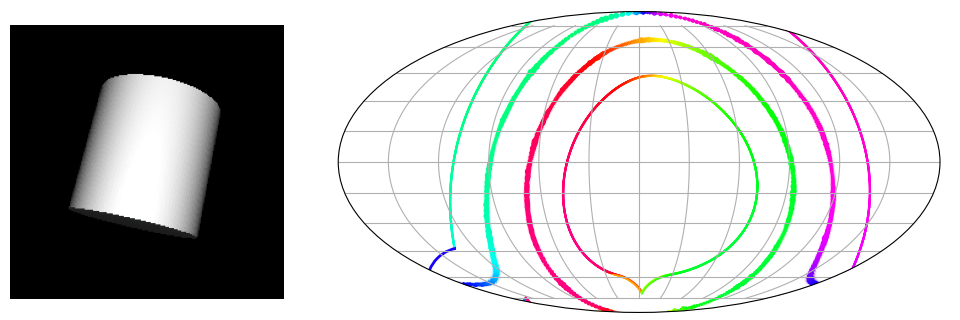}
     \end{subfigure}%
     \hfill
     \begin{subfigure}{.3\linewidth}
     \centering
     \includegraphics[width=1\textwidth]{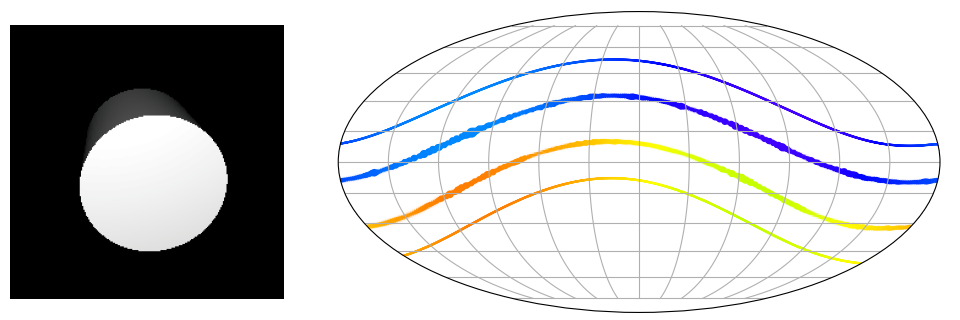}
     \end{subfigure}%
     \hfill
     \\
     \begin{subfigure}{.3\linewidth}
     \centering
     \includegraphics[width=1\textwidth]{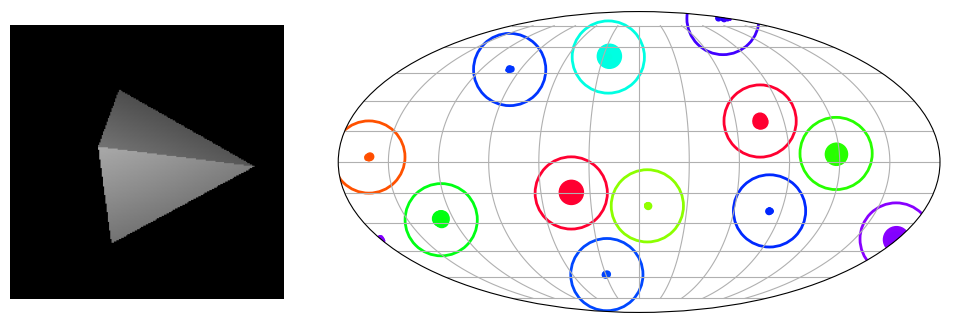}
     \end{subfigure}%
     \hfill
     \begin{subfigure}{.3\linewidth}
     \centering
     \includegraphics[width=1.\textwidth]{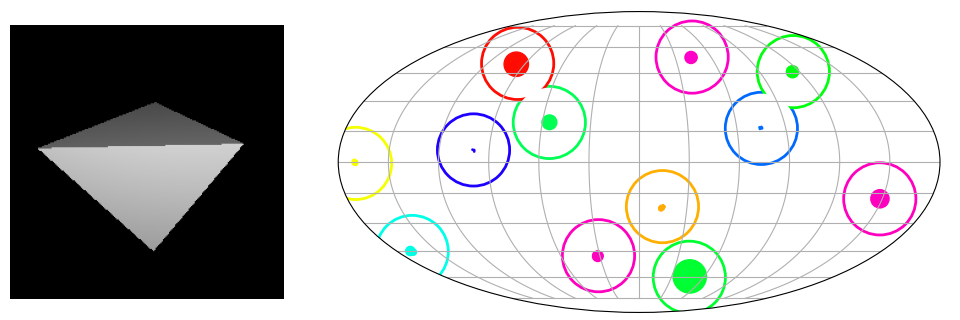}
     \end{subfigure}%
     \hfill
     \begin{subfigure}{.3\linewidth}
     \centering
     \includegraphics[width=1\textwidth]{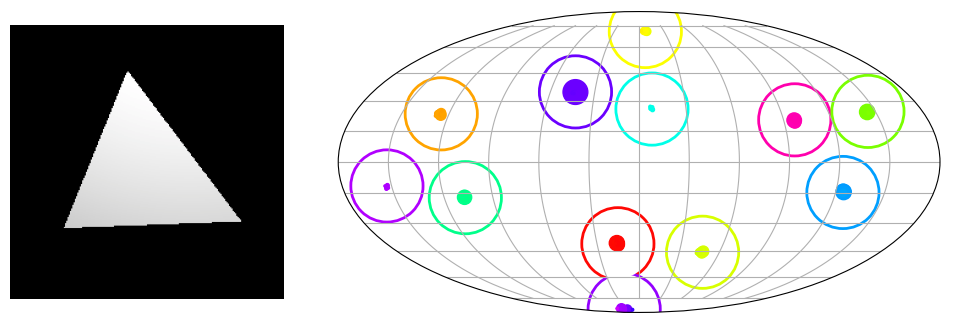}
     \end{subfigure}%
     \hfill
     \\
     \begin{subfigure}{.3\linewidth}
     \centering
     \includegraphics[width=1\textwidth]{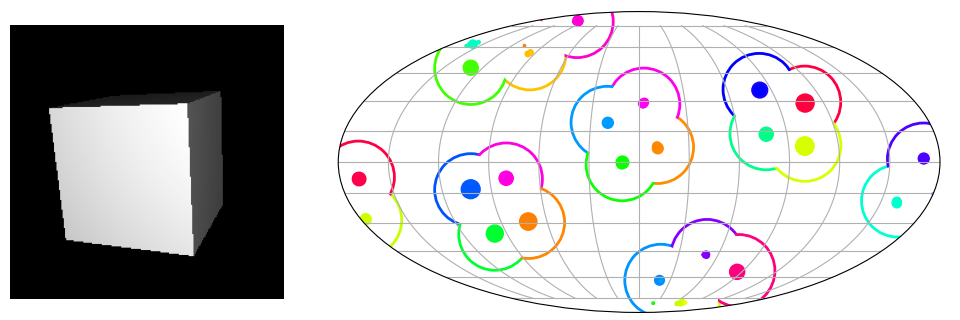}
     \end{subfigure}%
     \hfill
     \begin{subfigure}{.3\linewidth}
     \centering
     \includegraphics[width=1.\textwidth]{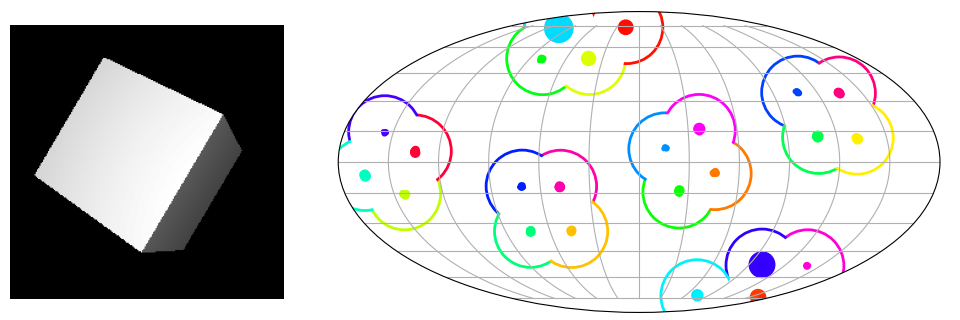}
     \end{subfigure}%
     \hfill
     \begin{subfigure}{.3\linewidth}
     \centering
     \includegraphics[width=1\textwidth]{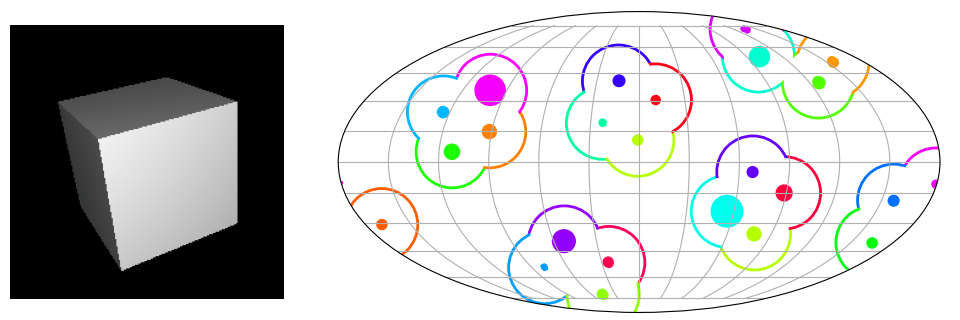}
     \end{subfigure}%
     \hfill
      \\
     \begin{subfigure}{.3\linewidth}
     \centering
     \includegraphics[width=1\textwidth]{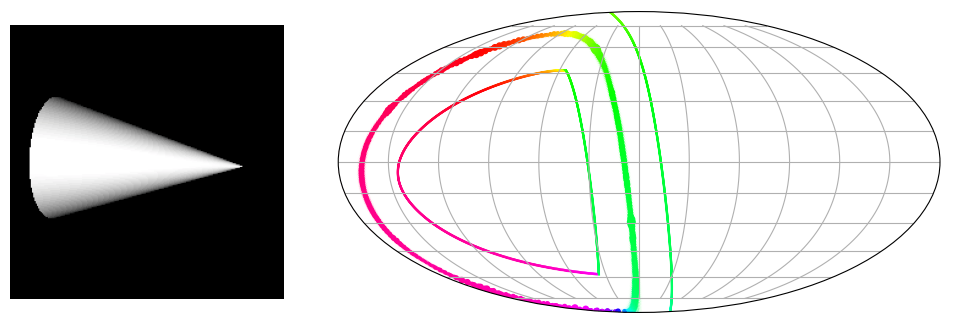}
     \end{subfigure}%
     \hfill
     \begin{subfigure}{.3\linewidth}
     \centering
     \includegraphics[width=1.\textwidth]{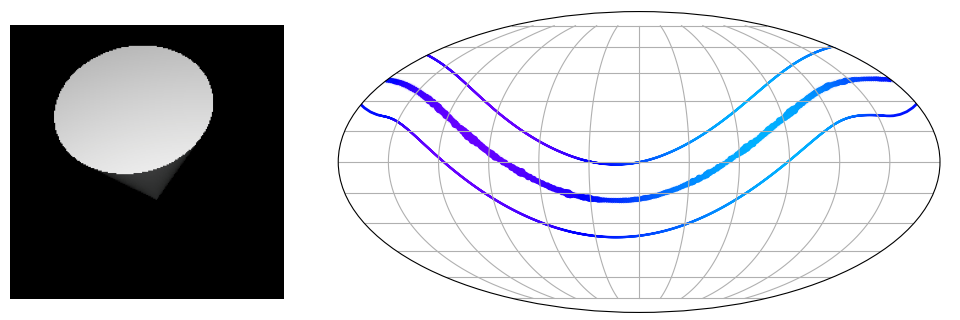}
     \end{subfigure}%
     \hfill
     \begin{subfigure}{.3\linewidth}
     \centering
     \includegraphics[width=1\textwidth]{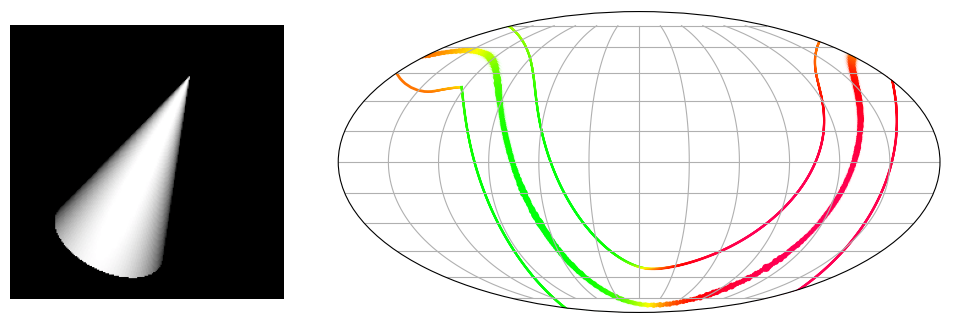}
     \end{subfigure}%
     \hfill
     %
      
     \caption{Pose distribution visualization for different objects in Symsol-I. Each row corresponds to a single object from Symsol. The distributions for cylinder, tetrahedron, cube, and cone objects are visualized. Cylinder and cone express continuous symmetries indicated by smooth curves, unlike tetrahedron and cube which have discrete modes.        
     \vspace{-4mm}
     }
     \label{fig:SymsolViz1}
\end{figure}

\vspace{2cm}

\begin{figure}[h]
     \centering
     \begin{subfigure}{.3\linewidth}
     \centering
     \includegraphics[width=1\textwidth]{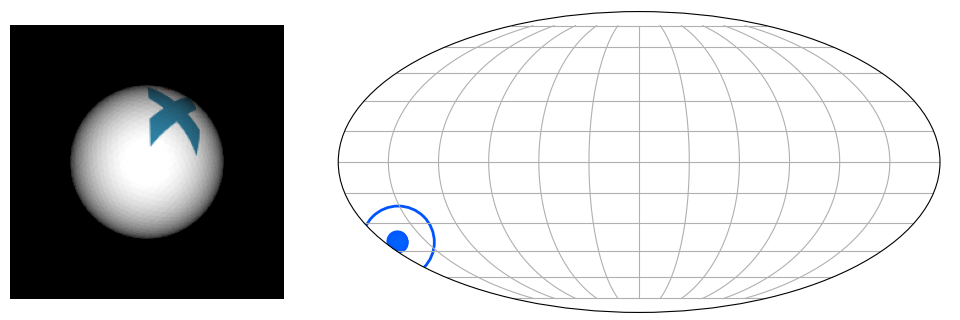}
     \end{subfigure}%
     \hfill
     \begin{subfigure}{.3\linewidth}
     \centering
     \includegraphics[width=1.\textwidth]{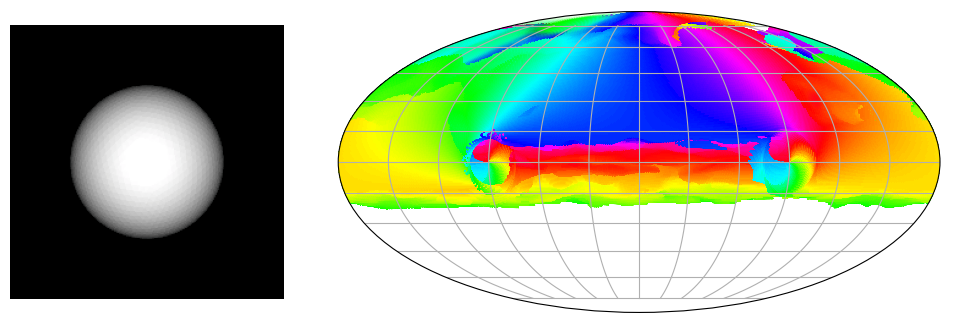}
     \end{subfigure}%
     \hfill
     \begin{subfigure}{.3\linewidth}
     \centering
     \includegraphics[width=1\textwidth]{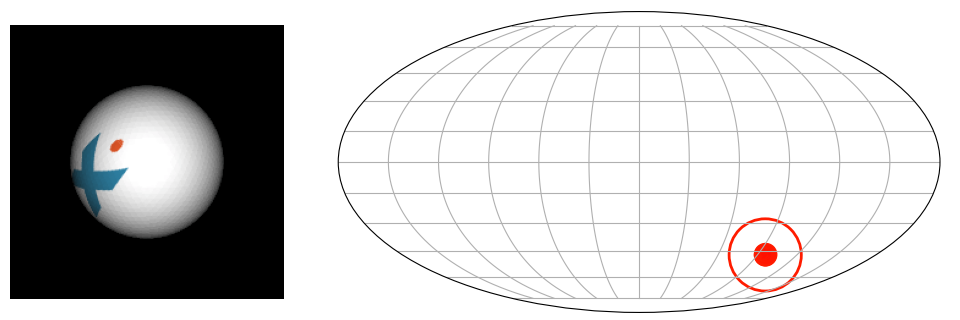}
     \end{subfigure}%
     \hfill
     \\
     \begin{subfigure}{.3\linewidth}
     \centering
     \includegraphics[width=1\textwidth]{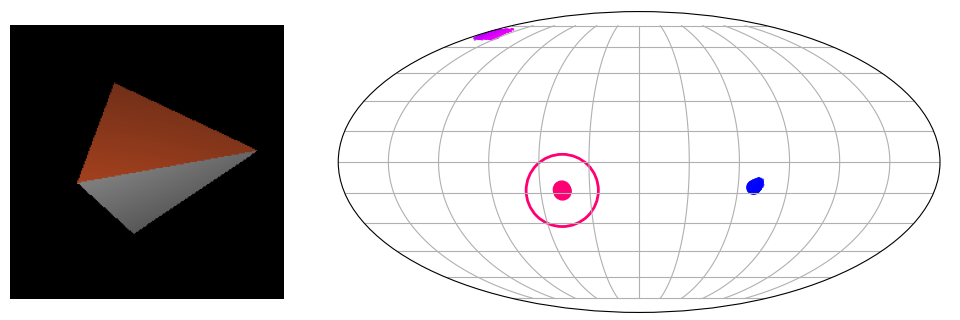}
     \end{subfigure}%
     \hfill
     \begin{subfigure}{.3\linewidth}
     \centering
     \includegraphics[width=1.\textwidth]{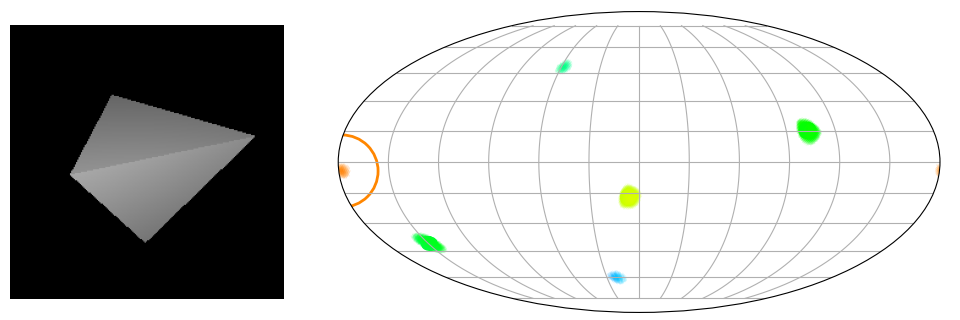}
     \end{subfigure}%
     \hfill
     \begin{subfigure}{.3\linewidth}
     \centering
     \includegraphics[width=1\textwidth]{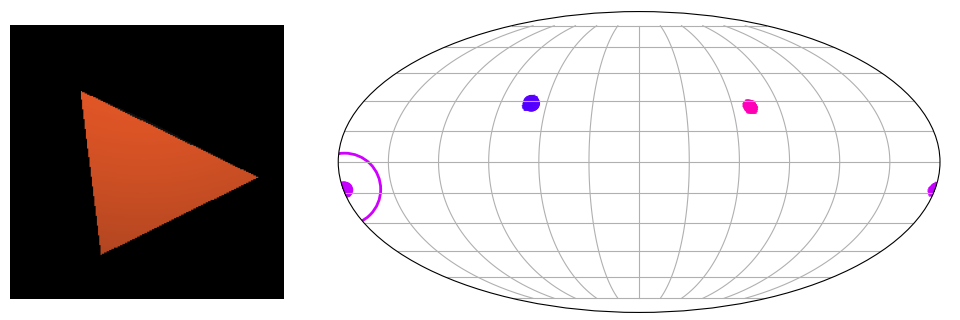}
     \end{subfigure}%
     \hfill
     \\
     \begin{subfigure}{.3\linewidth}
     \centering
     \includegraphics[width=1\textwidth]{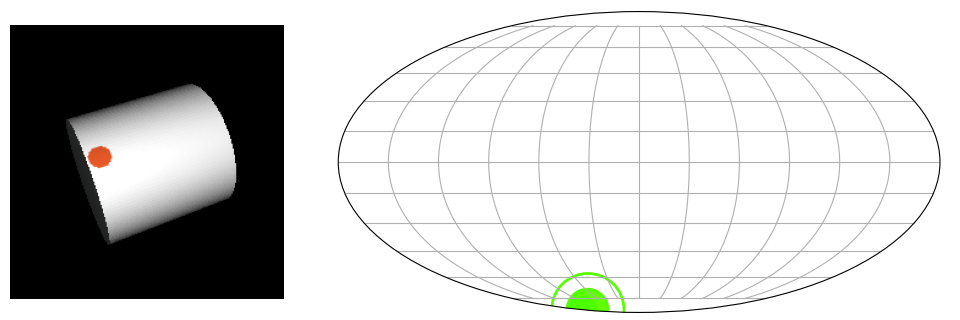}
     \end{subfigure}%
     \hfill
     \begin{subfigure}{.3\linewidth}
     \centering
     \includegraphics[width=1.\textwidth]{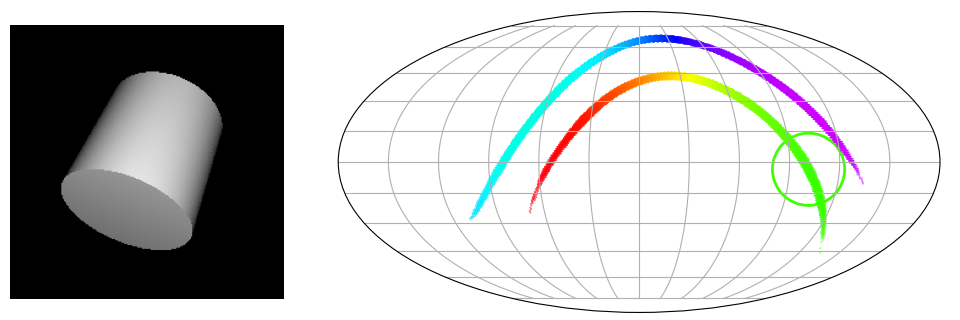}
     \end{subfigure}%
     \hfill
     \begin{subfigure}{.3\linewidth}
     \centering
     \includegraphics[width=1\textwidth]{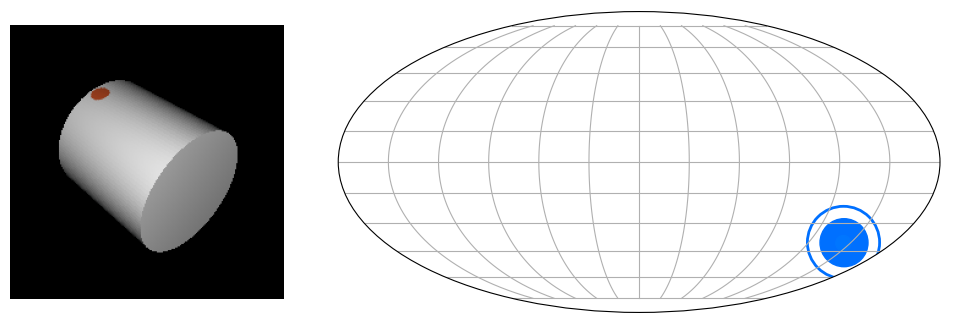}
     \end{subfigure}%
     \hfill
     \caption{Pose distribution visualization for different objects in Symsol-II. Each row corresponds to images from SphereX, TetX and CylO objects respectively. The middle column indicates the distribution when the markers are not visible. The left and right columns indicate the sharper distribution when the markers are visible. This shows that our approach can capture the distribution based on the texture component and make it sharper based on the marker when it is visible. This is possible because of the surfemb features which learn different features for textured regions and different features for untextured regions as shown in Figure \ref{fig:featViz}.
     \label{fig:SupSymsol2Viz}}
\end{figure}

\begin{figure}
     \centering
     \begin{subfigure}{.5\linewidth}
     \centering
     \includegraphics[width=1\textwidth]{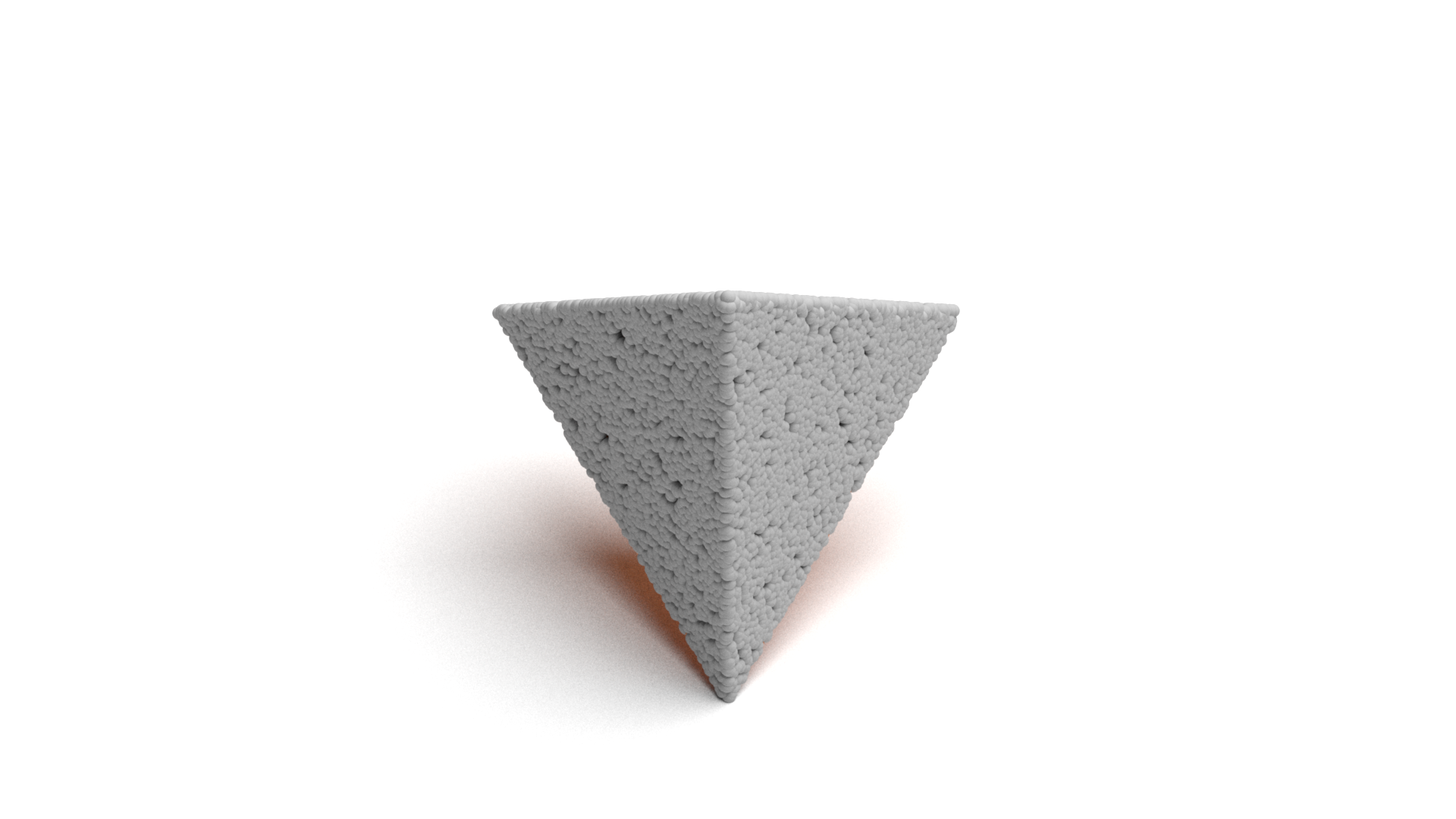}
     \end{subfigure}%
     \hfill
     \begin{subfigure}{.5\linewidth}
     \centering
     \includegraphics[width=1.\textwidth]{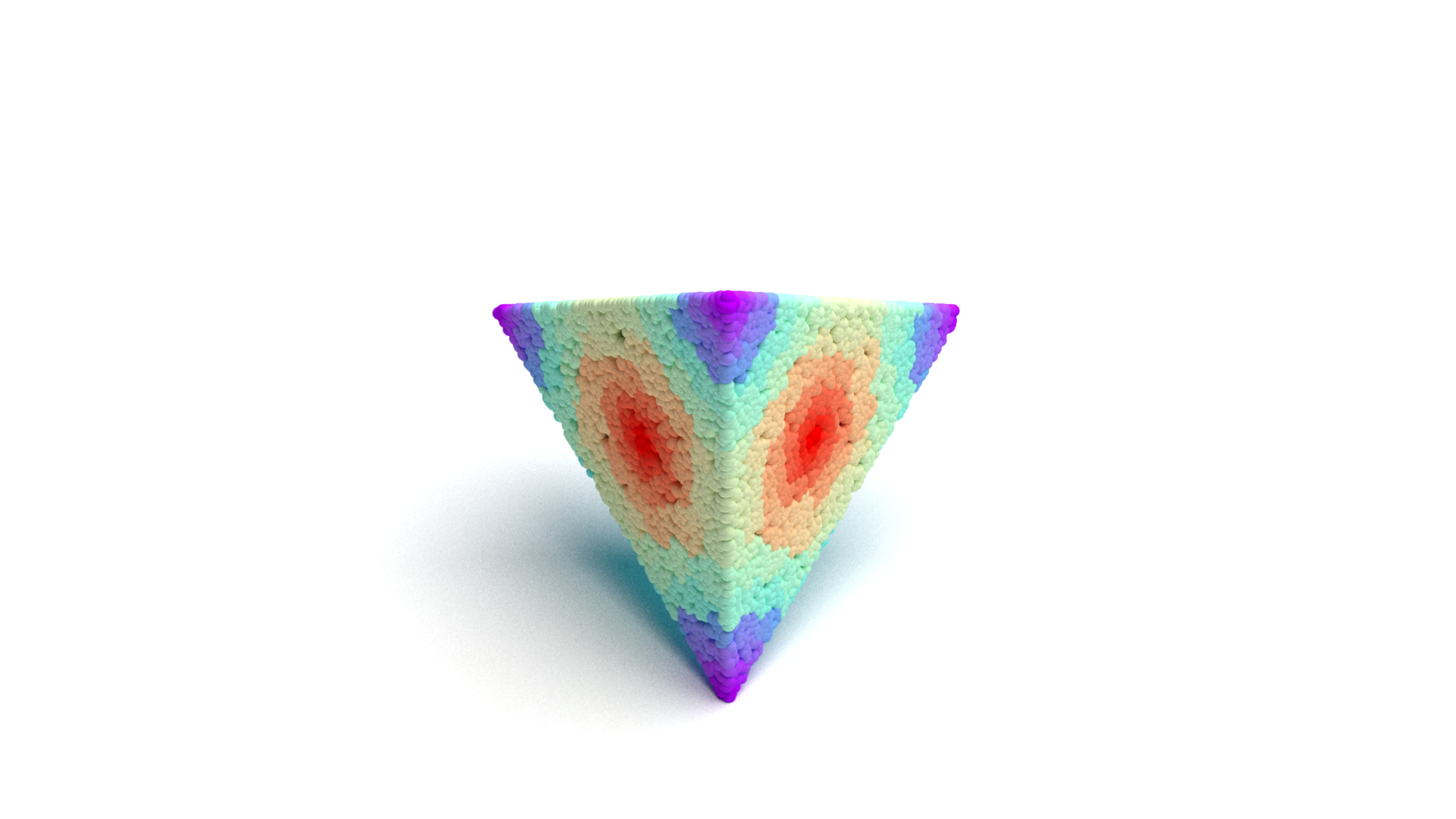}
     \end{subfigure}%
     \hfill
     \\
     \begin{subfigure}{.5\linewidth}
     \centering
     \includegraphics[width=1\textwidth]{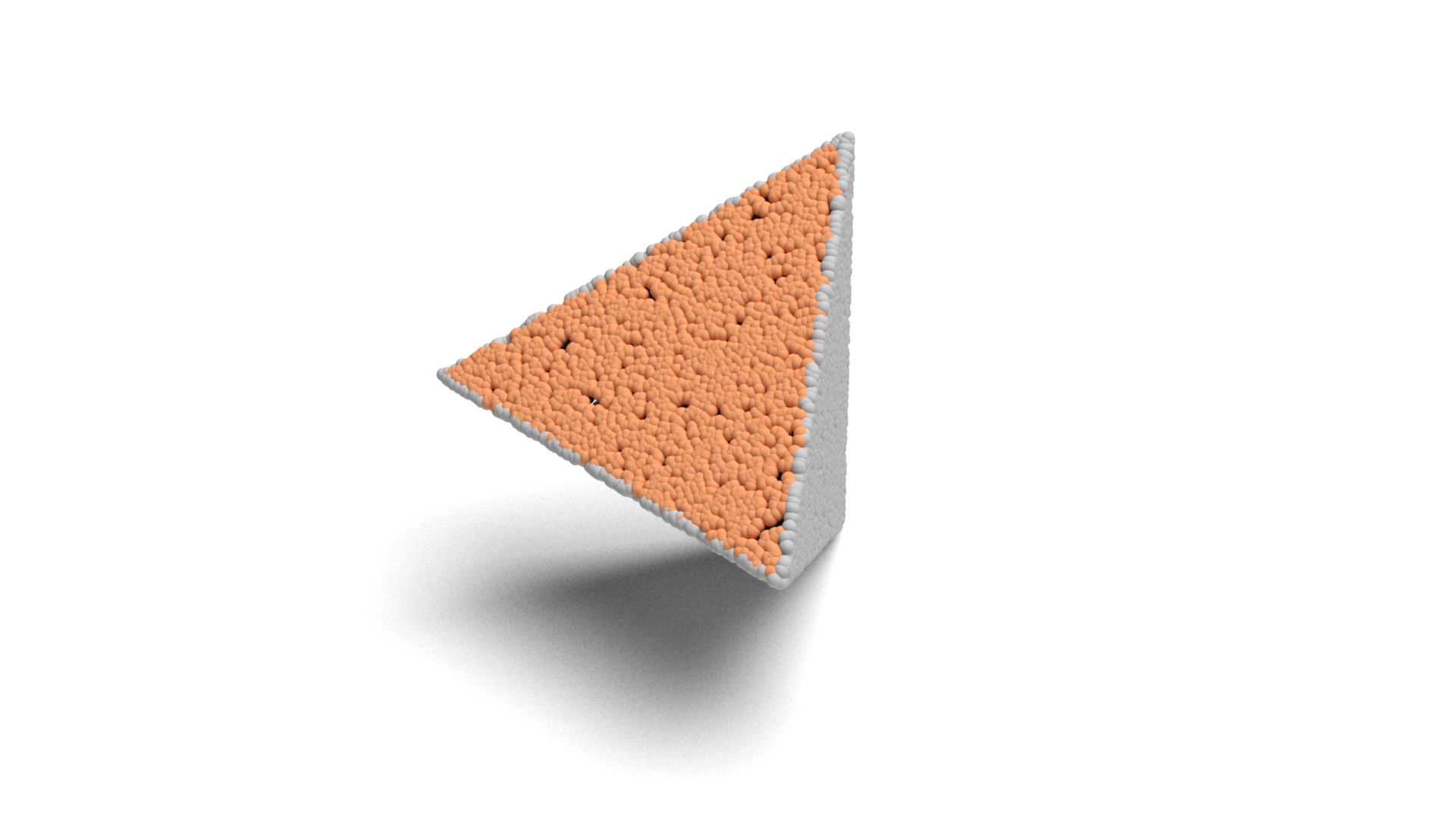}
     \end{subfigure}%
     \hfill
     \begin{subfigure}{.5\linewidth}
     \centering
     \includegraphics[width=1.\textwidth]{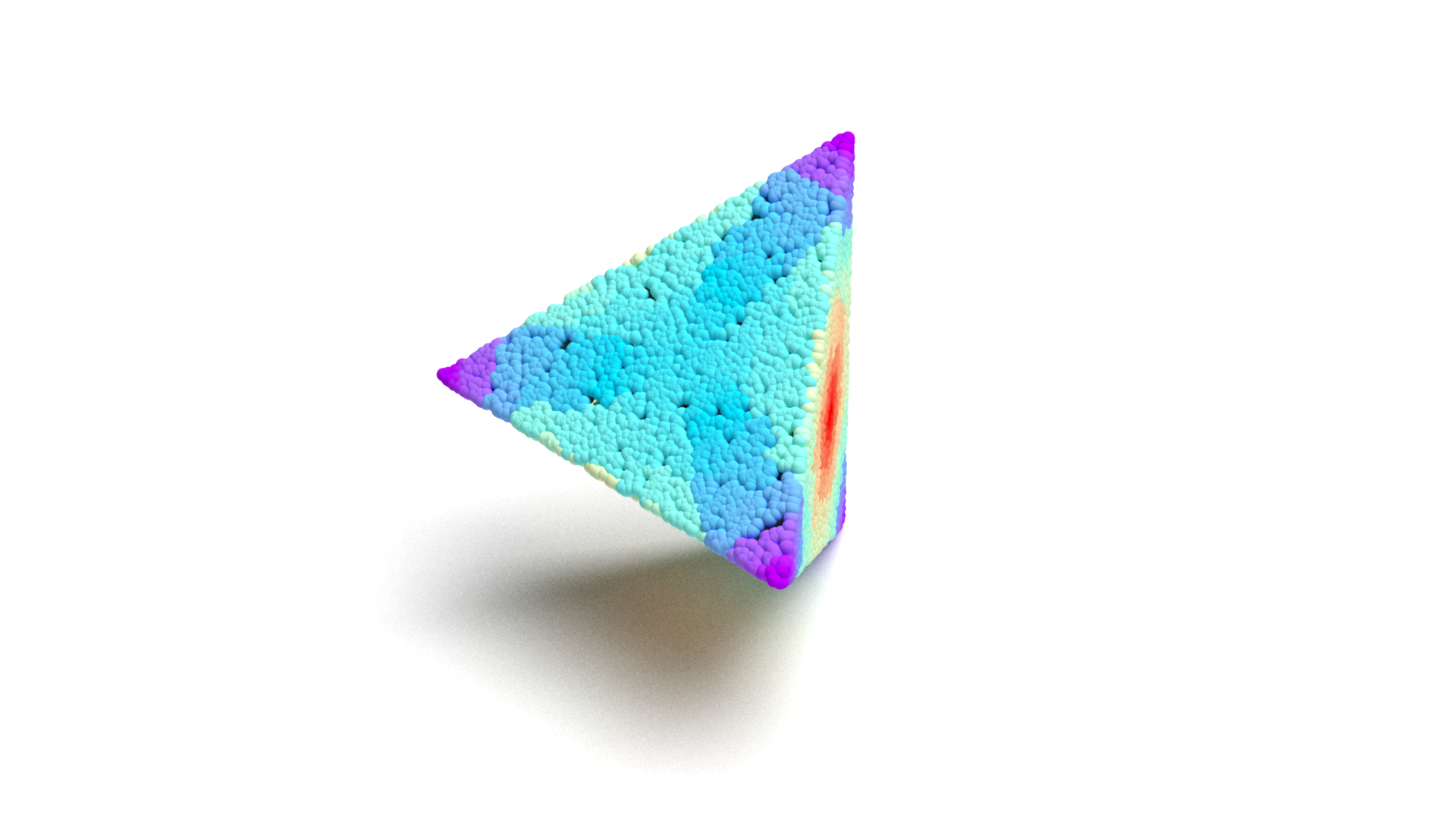}
     \end{subfigure}%
     \hfill
     \hfill
\\
     \begin{subfigure}{.5\linewidth}
     \centering
     \includegraphics[width=1\textwidth]{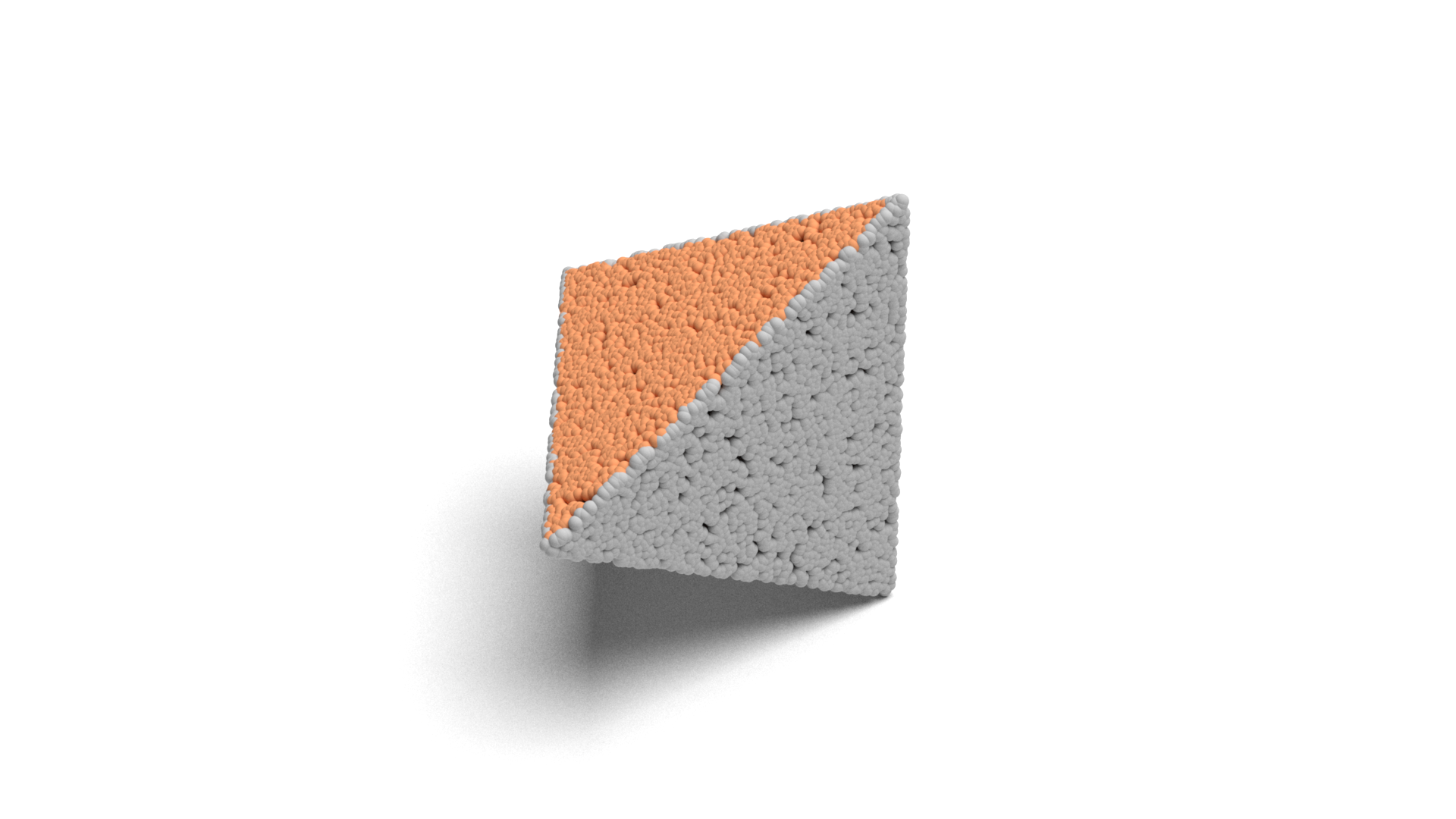}
     \end{subfigure}%
     \hfill
     \begin{subfigure}{.5\linewidth}
     \centering
     \includegraphics[width=1.\textwidth]{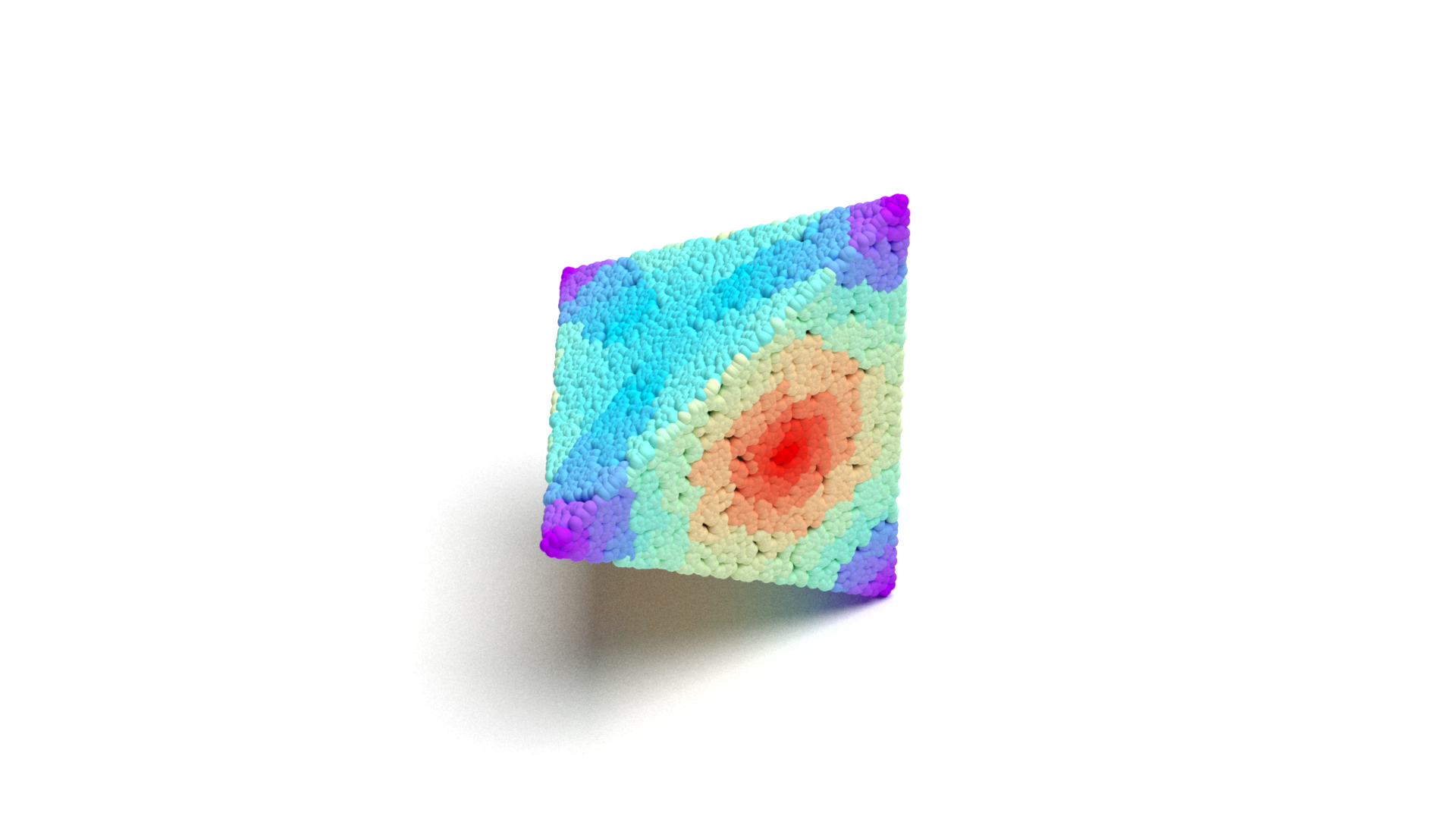}
     \end{subfigure}%
     \hfill\\
     \begin{subfigure}{.5\linewidth}
     \centering
     \includegraphics[width=1\textwidth]{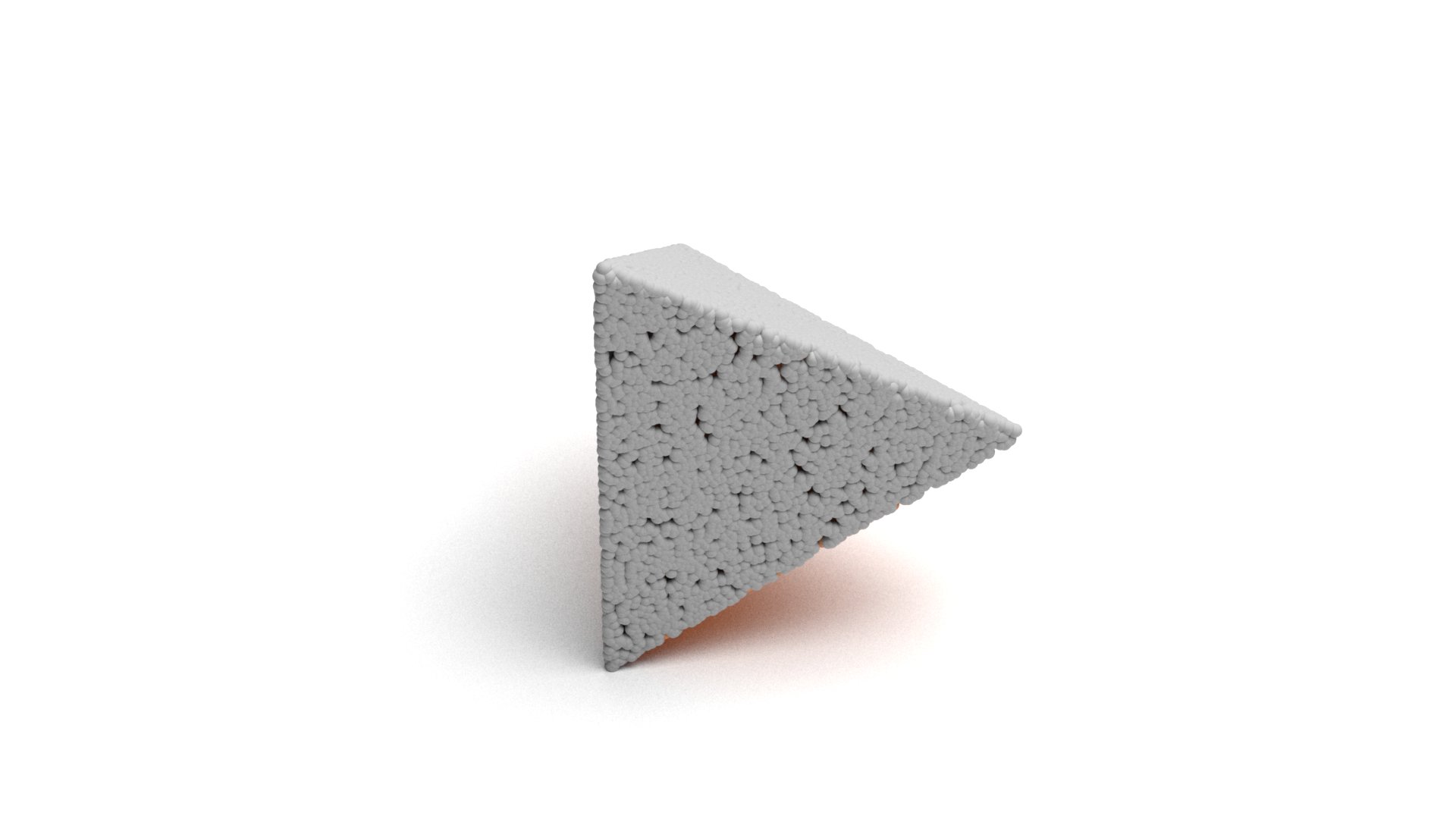}
     \end{subfigure}%
     \hfill
     \begin{subfigure}{.5\linewidth}
     \centering
     \includegraphics[width=1.\textwidth]{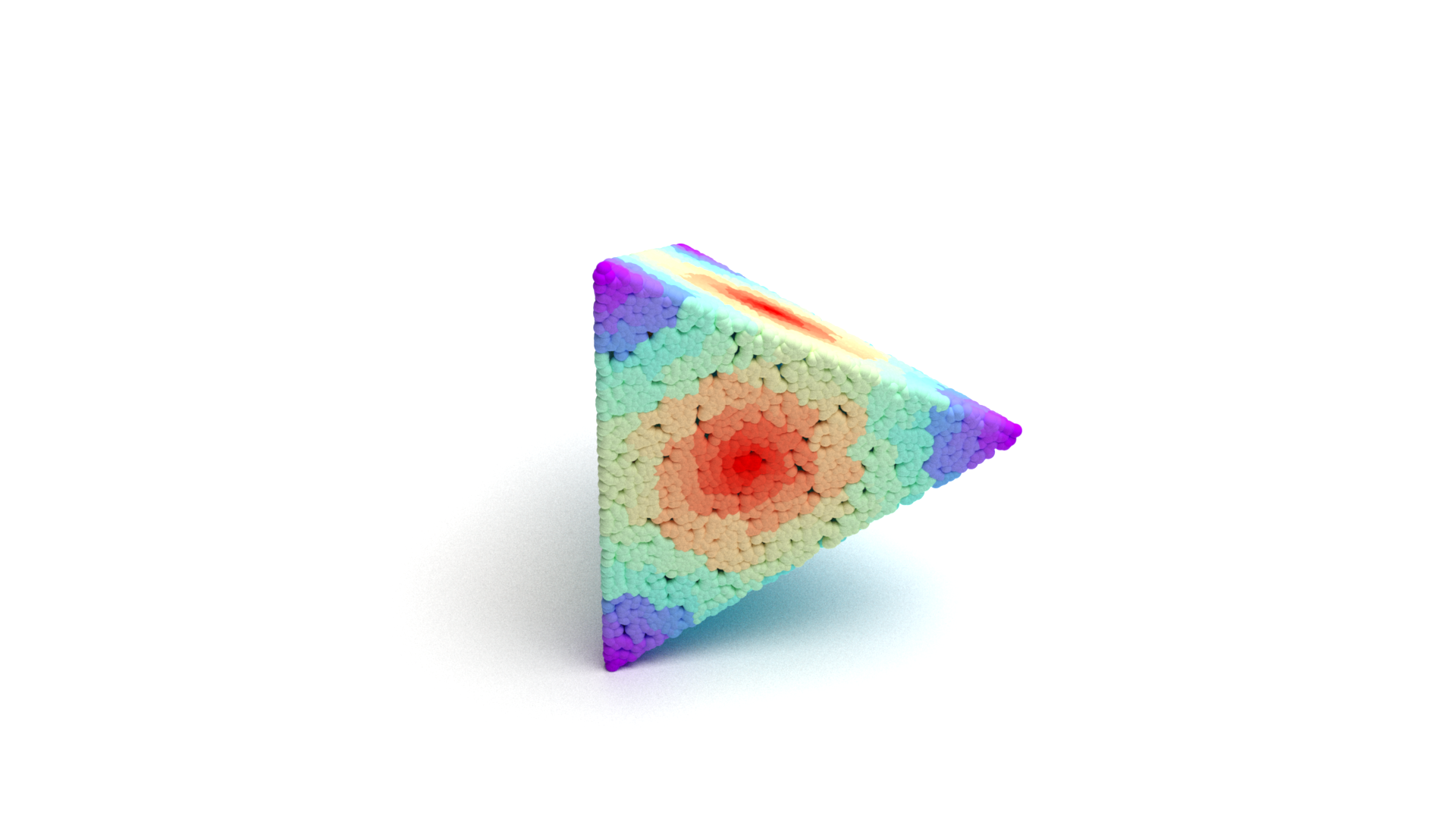}
     \end{subfigure}%
     \hfill
     \caption{Feature visualization of SurfEmb features for tetX object where a face of tetrahedron is textured with orange color. The first column shows the point cloud from the textured CAD model. The second column shows the per-point feature visualization of the point cloud from different viewpoints. Note that the features are computed in canonical orientation. The visualization indicates different rotations to illustrate the difference in features on the orange face and non-textured faces. The features are clearly different on the non-textured sides compared to the textured side. Also, all of the non-textured sides have similar features.    
     \label{fig:featViz}}
\end{figure}
\begin{figure}
     \centering
     \begin{subfigure}{.3\linewidth}
     \centering
     \includegraphics[width=0.8\textwidth]{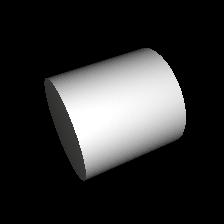}
     \end{subfigure}%
     \hfill
     \begin{subfigure}{.3\linewidth}
     \centering
     \includegraphics[width=.8\textwidth]{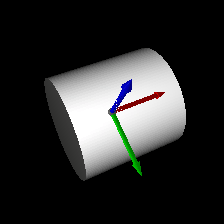}
     \end{subfigure}%
     \hfill
     \begin{subfigure}{.3\linewidth}
     \centering
     \includegraphics[width=0.8\textwidth]{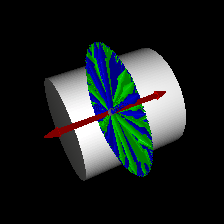}
     \end{subfigure}%
     \hfill
     \\
     \begin{subfigure}{.3\linewidth}
     \centering
     \includegraphics[width=.8\textwidth]{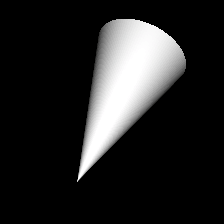}
     \end{subfigure}%
     \hfill
     \begin{subfigure}{.3\linewidth}
     \centering
     \includegraphics[width=0.8\textwidth]{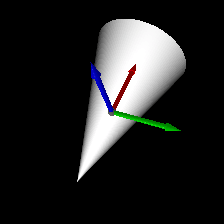}
     \end{subfigure}%
     \hfill
     \begin{subfigure}{.3\linewidth}
     \centering
     \includegraphics[width=0.8\textwidth]{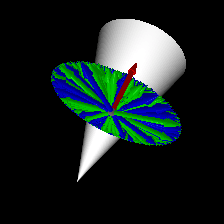}
     \end{subfigure}%
     \hfill
     \\
     \begin{subfigure}{.3\linewidth}
     \centering
     \includegraphics[width=0.8\textwidth]{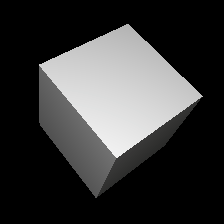}
     \end{subfigure}%
     \hfill
     \begin{subfigure}{.3\linewidth}
     \centering
     \includegraphics[width=.8\textwidth]{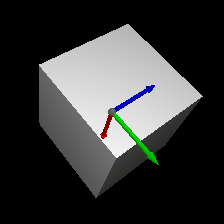}
     \end{subfigure}%
     \hfill
     \begin{subfigure}{.3\linewidth}
     \centering
     \includegraphics[width=.8\textwidth]{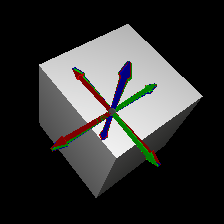}
     \end{subfigure}%
     \hfill
      \\
     \begin{subfigure}{.3\linewidth}
     \centering
     \includegraphics[width=.8\textwidth]{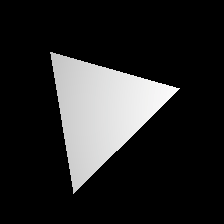}
     \end{subfigure}%
     \hfill
     \begin{subfigure}{.3\linewidth}
     \centering
     \includegraphics[width=.8\textwidth]{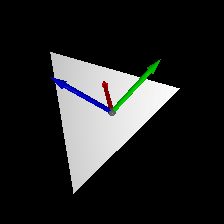}
     \end{subfigure}%
     \hfill
     \begin{subfigure}{.3\linewidth}
     \centering
     \includegraphics[width=.8\textwidth]{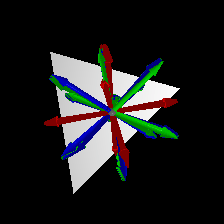}
     \end{subfigure}%
     \hfill
      \\
     \begin{subfigure}{.3\linewidth}
     \centering
     \includegraphics[width=.8\textwidth]{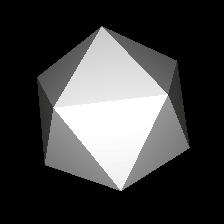}
     \end{subfigure}%
     \hfill
     \begin{subfigure}{.3\linewidth}
     \centering
     \includegraphics[width=.8\textwidth]{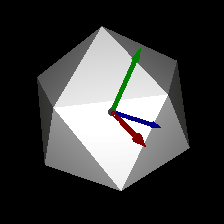}
     \end{subfigure}%
     \hfill
     \begin{subfigure}{.3\linewidth}
     \centering
     \includegraphics[width=.8\textwidth]{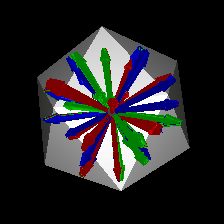}
     \end{subfigure}%
     \hfill
     %
      
     \caption{Pose distribution visualization for different objects in Symsol-I. Each row corresponds to a single object starting from a Cylinder, Cone, Cube, Tetrahedron, and Icosa. The first column shows the input image. The second column shows the coordinate reference frame transformed with ground truth rotation and rendered on the input image. In the third column, we transform the coordinate reference frame with top 500 rotations in the distribution and render on the image. Our approach captures the distribution sharply and captures all the modes. This is clearly indicated in cylinder and cone where the red axis is very sharp without much deviation which indicates that the symmetric axis is correctly detected and the variation around the symmetric axis is properly captured in top 500 rotations instead of focusing on some specific modes. The sharpness is also clearly visible in cube, tetrahedron, and icosa where the axes have less blur around them.   
     \label{fig:ourViz}}
\end{figure}

\clearpage

\end{document}